\documentclass{article}

\usepackage{microtype}
\usepackage{graphicx}
\usepackage{subfigure}
\usepackage{booktabs} 

\usepackage{hyperref}

\usepackage[accepted]{template_arxiv}

\usepackage{amsmath}
\usepackage{amssymb}
\usepackage{mathtools}
\usepackage{amsthm}

\usepackage[capitalize,noabbrev]{cleveref}

\theoremstyle{plain}
\newtheorem{theorem}{Theorem}[section]
\newtheorem{proposition}[theorem]{Proposition}
\newtheorem{lemma}[theorem]{Lemma}
\newtheorem{corollary}[theorem]{Corollary}
\theoremstyle{definition}

\newtheorem{assumption}[theorem]{Assumption}
\theoremstyle{remark}

\usepackage[textsize=tiny]{todonotes}

\usepackage{graphicx} 
\usepackage{color}
\usepackage{hyperref}
\usepackage{amsmath}
\usepackage{dsfont}
\usepackage{amssymb}
\usepackage{amsthm}
\usepackage{mathtools}
\usepackage{algorithm}
\usepackage{algorithmic}
\usepackage{xcolor}
\usepackage{natbib}
\usepackage{cleveref}
\DeclareMathOperator{\kl}{kl}

\DeclareMathOperator{\KL}{KL}

\renewcommand*\Pr{\mathbb{P}}

\usepackage{todonotes}

\icmltitlerunning{Clustering Items through Bandit Feedback: Finding the Right Feature out of Many}

\begin{document}

\onecolumn 
\icmltitle{Clustering Items through Bandit Feedback\\ Finding the Right Feature out of Many}



\icmlsetsymbol{equal}{*}

\begin{icmlauthorlist}
\icmlauthor{Maximilian Graf}{equal,yyy}
\icmlauthor{Victor Thuot}{equal,xxx}
\icmlauthor{Nicolas Verzelen}{xxx}
\end{icmlauthorlist}

\icmlaffiliation{yyy}{Institut für Mathematik, Universität Potsdam, Potsdam, Germany}
\icmlaffiliation{xxx}{INRAE, Mistea, Institut Agro, Univ Montpellier, Montpellier, France}

\icmlcorrespondingauthor{Maximilian Graf}{graf9@uni-potsdam.de}
\icmlcorrespondingauthor{Victor Thuot}{victor.thuot@inrae.fr}

\icmlkeywords{clustering, bandit theory, pure exploration, information-theoretic bounds, machine learning}

\vskip 0.3in



\printAffiliationsAndNotice{ \icmlEqualContribution} 

\begin{abstract}
We study the problem of clustering a set of items based on bandit feedback. Each of the $n$ items is characterized by a feature vector, with a possibly large dimension $d$. The items are  partitioned into two unknown groups, such that items within the same group share the same feature vector. We consider a sequential and adaptive setting in which, at each round, the learner selects one item and one feature, then observes a noisy evaluation of the item's feature. The learner's objective is to recover the correct partition of the items, while keeping the number of observations as small as possible. 
We provide an algorithm which relies on finding a relevant feature for the clustering task, leveraging the Sequential Halving algorithm. With probability at least $1-\delta$, we obtain an accurate recovery of the partition and derive an upper bound on the budget required. Furthermore, we derive an instance-dependent lower bound, which is tight in some relevant cases.
\end{abstract}

\section{Introduction}

We consider a sequential and adaptive pure exploration problem, in which a learner aims at clustering a set of items, each represented by a feature vector. The items are partitioned into two unknown groups such that items within the same group share the same feature vector. The learner sequentially selects an item and a feature, then observes a noisy evaluation of the chosen feature for that item. Given a prescribed probability $\delta$, the objective of the learner is to collect enough information to recover the partition of the items with a probability of error smaller than $\delta$. 

This problem appears in crowd-sourcing platforms, where complex labeling tasks are decomposed into simpler sub-tasks, typically involving answering specific questions about an item, see~\cite{ariu2024optimal}. A motivating example is image labeling: a platform sequentially presents an image to a user along with a simple question such as "Is this a vehicle?" or "How many wheels can you see?". The learner leverages these answers to classify the images into two categories. In this setting, the images corresponds to items that have to be clustered, while questions corresponds to features. This problem is a special case of the model studied in~\cite{ariu2024optimal}, where the authors numerically demonstrate the advantage of an adaptive sampling scheme over non-adaptive ones. However, they do not the establish the theoretical validity of their procedure.

From a theoretical perspective, our problem consists of clustering $n$ items from sequential and adaptive queries to some of their $d$ features. Intuitively, the difficulty of the clustering task is quantified by the difference between items in different groups on these  $d$ features. In particular, the difficulty depends both on the magnitude of these differences and on their sparsity, that is, the number of feature on which the items differ. 

In this work, we precisely characterize the sample complexity of the clustering task, in a fully adaptive way. Our {\bf main contributions} are as follows: 
\begin{itemize}\vspace{-0.3cm}
    \item We introduce the BanditClustering  procedure -- \cref{alg:cluster}. On the one hand, it outputs the correct partition of the items with a prescribed probability $1-\delta$. On the other hand, it adapts to the unknown means of the group in order to sample at most the informative positions. In \cref{thm:main}, we provide a tight non-asymptotic bound on its budget as a function of $n$, $d$, $\log(1/\delta)$, and the difference between the means.
    \vspace{-0.3cm}
    \item Conversely, we establish in Section~\ref{section:LB} theoretical-information lower bound on the budget which entail the optimality of BanditClustering. 
    \vspace{-0.3cm}
\end{itemize}

From a broader perspective, our algorithm operates in three steps: first, it identifies two items from distinct groups; second, it finds a suitable feature that best discriminates between the two items; and finally, it builds upon this specific feature to cluster all the items

{\bf Connection to good arm identification and adaptive sensing literature}. One of the key challenges is to achieve the trade-off between the budget for the identification of a good discriminative feature and budget for the clustering task. We borrow techniques from the best arm identification literature: specifically, we employ the Sequential Halving algorithm \cite{pmlr-v28-karnin13} as a subroutine, leveraging its strong performance in scenarios where multiple arms are nearly optimal~\cite{zhao2023revisiting,pmlr-v108-katz-samuels20a,chaudhuri2019pac}. The first identification step -- finding two items belonging to distinct groups -- is closely related to the adaptive sensing strategies for signal detection as studied in \cite{castro2014adaptive}, where the problem is framed as a sequential and adaptive hypothesis testing task. Additionally, our approach incorporates ideas from \cite{castro2014adaptive,saad2023active} to efficiently identify the most informative features for clustering.

{\bf Connection to dueling bandits literature}. Our bandit clustering problem is an instance of a pure bandit exploration problem where one sample interaction between items and features. In that respect, it is also related to ranking \cite{saad2023active} and dueling bandits in the online literature \citep{ailon2014reducing,chen2020combinatorial,heckel2019active,jamieson2011active,jamieson2015sparse,urvoy2013generic,yue2012k,haddenhorst2021identification}  where the goal is to recover a partition of the items based on noisy pairwise comparisons. Some of the ranking procedures are based on estimation of the Borda count~\cite{heckel2019active}, some other procedures as~\cite{saad2023active} aim at adapting to the unknown form of the comparison matrix to lower the total budget. In essence, our approach is related, as it seeks to balance the trade-off between identifying relevant entries and exploiting them for efficient comparison.

{\bf Connection to other bandit clustering problems}. 
Recent works~\cite{yang2024optimal,thuot2024active,yavas10838574} have investigated clustering in a bandit setting, where items must be clustered based on noisy evaluations of their feature vectors.
However, in these settings, the entire feature vector of an item is observed at each sampling step, whereas in our framework, only a single feature of a given item is observed per step. Our observation scheme enables a more efficient allocation of the budget by focusing on the most relevant features—those that best discriminate between groups. The trade-off between exploration of the relevant feature and exploitation for classifying items is at the core of our work. This allows us to cluster the items with a much lower observation budget than in~\cite{yang2024optimal,thuot2024active} -- see the discussion section. Other authors previously introduced adaptive clustering problem for crowd-sourcing \cite{ho2013adaptive,gomes2011crowdclustering}, although their setting does not directly relate to ours.

{\bf Organization of the manuscript}: The model is introduced together with notation in the following \cref{section:model}. In Section~\ref{section:UB}, we describe our procedure and control its budget. Section~\ref{section:LB} provides matching lower bounds of the budget that entail the optimality of our procedure. Finally, we provide some numerical experiments  in Section~\ref{sec:experiments} and discuss possible extensions in Section~\ref{sec:discussion}.

\section{Problem formulation and notation}\label{section:model}

Consider a set of $n$ items, indexed by $[n]$. Each item is characterized by a feature vector of dimension $d$, where the number of features $d$ may be large. Let $M$ be the $n \times d$ matrix such that the $i$-th row of $M$ contains the feature vector of item $i$. We denote the feature vector of item $i$ as $\mu_i = (M_{i,1}, \dots, M_{i,d})$. We assume that the $n$ items are partitioned into two unknown groups, such that items within the same group share the same feature vector. The groups are assumed to be nonempty and non-overlapping. The objective is to recover these two groups. 

\begin{assumption}[Hidden partition]\label{assumption:hidden_partition} 
There exists two different vectors $\mu^a\in \mathbb R^d$ and $\mu^b \in \mathbb R^d$, and a (non constant) label vector $g\in\{0,1\}^n$ such that for any item $i\in[n]$, 
\begin{align*}
M_{i,j}=
    \begin{cases}
       \mu^a_{j} & \text{if } g(i) = 0 \enspace, \\
       \mu^b_{j} & \text{if } g(i) = 1 \enspace .
    \end{cases} 
\end{align*}
\end{assumption}

As usual in clustering,  $(g,\mu^{a},\mu^{b})$ encodes the same matrix $M$ as $(1-g,\mu^{b},\mu^{a})$. 
This is why, without loss of generality, we further constrain that $g(1)=0$ to make the label vector $g$ identifiable. 

We consider a bandit setting, where the learner sequentially and adaptively observes noisy  entries of the matrix  $M$. At each time step $t$, based on passed observation, the learner selects one item $I_t\in[n]$ and one question $J_t\in[d]$, then she receives $X_t$, a noisy evaluation of $M_{I_t,J_t}$. Conditionally on the couple $(I_t,J_t)$, $X_t$ is a new and independent random variable following the unknown distribution $\nu_{I_t,J_t}$, with expectation $M_{I_t,J_t}$. The set of distributions $(\nu_{i,j})_{i,j}$ is referred to as the environment.

We are going to assume that the noise in observations is $1$-subGaussian.
\begin{assumption}[$1$-subGaussian noise]\label{assumption:subGaussian} 
For any pair $(i,j) \in [n]\times[d]$, if $X\sim \nu_{i,j}$, then $X-M_{i,j}$ is $1$-subGaussian, namely
\begin{align*}
    \mathbb E\left[\exp(t(X-M_{i,j})\right]\leq \exp\left(t^2/2\right) \quad \forall t\in \mathbb{R}\enspace .
\end{align*}
\end{assumption}

The subGaussian assumption is standard in the bandit literature~\cite{lattimore2020bandit}. It encompasses the emblematic setting where the data follow Gaussian distributions with variance bounded by $1$. It also covers the case of random variables bounded in $[0,1]$ such as Bernoulli distributions. By rescaling, we could extend this to $\sigma$-subGaussian noise.  

We tackle this pure exploration problem in the \emph{fixed confidence setting}, a common framework in bandit literature~\cite{lattimore2020bandit}. In this setting, the learner also decides when to stop learning. For a prescribed probability of error $\delta$, the learner's objective is to recover the groups with a probability of error not exceeding $\delta$, while minimizing the number of observations. The learner samples observations up to a stopping time $T$, and then outputs $\hat{g}\in\{0,1\}^n$ [with $\hat{g}(1)=0$], which estimates the labels $g$. The termination time $T$ corresponds to the total number of observations collected. Hence, we refer to $T$ as the budget of the procedure. Formally, $T$ is a stopping time according to the natural filtration associated to the sequential model.

For any confidence $\delta$, we say then that a procedure $\mathcal{A}$ is $\delta$-PAC if 
\begin{equation*}
    \mathbb{P}_{\mathcal{A},\nu}(\hat{g}= g)\geqslant 1-\delta \enspace,
\end{equation*}
where the probability $\mathbb{P}_{\mathcal{A},\nu}$ is induced by the interaction between environment $\nu$ and algorithm $\mathcal{A}$.

 The performance of any $\delta$-PAC algorithm is measured with the budget $T$, which should be as small as possible. In this paper, we provide upper bounds on $T$ that holds with high probability, typically in the event of probability larger than $1-\delta$ on which the algorithm is correct. 

We introduce two key quantities for analyzing the problem. The \emph{gap vector} is defined as
 $$\Delta:=\mu^a-\mu^b \enspace,$$
this quantity naturally appears to characterize the difficulty of the clustering task. By assumption, we have $\Delta\ne 0$. The smaller the norm of $\Delta$ is, the more challenging the estimation of $g$ is.
In particular, we analyze the complexity of the clustering task with respect to the entry of the gap vector $\Delta$ ordered by decreasing absolute value, namely
\begin{equation*}
    \left|\Delta_{(1)}\right|\geq \left|\Delta_{(2)}\right|\geq \dots \geq \left|\Delta_{(d)}\right|\enspace .
\end{equation*}

Besides, we define $\theta$ the balancedness of the partition $g$, that is the proportion of arms in the smallest group 
$$\theta \coloneqq \frac{\sum_{i=1}^n\mathbf{1}(g(i)=0)\wedge \sum_{i=1}^n\mathbf{1}(g(i)=1)}{n} \enspace.$$

Intuitively, the smaller $\theta$ is, the more unbalanced the partition is and the more difficult it is to discover two items of distinct groups. For identifiability reasons, we assumed in \cref{assumption:hidden_partition} that the groups are nonempty -- which implies in particular that $n\geqslant 2$, but also that $\theta\in [1/n; 1/2]$.

\section{Algorithms}\label{section:UB}

\subsection{Introduction to our method}

Our method first explores the matrix to identify a feature on which the two groups of items differ significantly. Without loss of generality, we use the first item as a representative item from the first group (since we fixed $\mu_1=\mu^a$). The primary goal is to identify an item $j\in\{2,\dots,n\}$ and a feature $j\in[d]$ such that $\mu_{ij}$ differs from $\mu_{1j}$ significantly, that is such that $|\mu_{ij}-\mu_{1j}|$ is large. 

If such entry $(i,j)$ is detected, we can estimate $|\Delta_j|$, and use it to classify  the items based on the feature $j$, with a budget scaling as $O\left(\frac{n}{\Delta_j^2}\log(n/\delta)\right)$. Our method balances the budget spent on identifying a discriminative feature with the budget required for this final classification step.

Importantly, if we detect an entry $(i,j)$ in the matrix where the gap $|\mu_{ij}-\mu_{ij}|$ is sufficiently large, we collect two di-symmetric pieces of information: 
\begin{itemize}
    \vspace{-0.3cm}
    \item If determine that $|\mu_{ij}-\mu_{1j}|$ is strictly positive, then item $i$ serves as a representative of the second group, allowing us to learn $\mu^b$ (hence $\Delta$) by sampling from item $i$. 
    \vspace{-0.2cm}
    \item If we establish not only that $|\mu_{ij}-\mu_{1j}|$ is nonzero but also that it exceeds a certain threshold-- which we specify later -- then feature $j$ is discriminative enough. At this point, we can concentrate the classification budget on this feature. 
    \vspace{-0.3cm}
\end{itemize}

This observation is central to our method, leading us to structure the clustering task as a two-step procedure:
\begin{enumerate}
    \vspace{-0.3cm}
    \item First, we identify an item from the second group, we provide \cref{alg:cr} that identifies an item $\hat{i}_c$, such that $\mu_{\hat{i}_c}\ne \mu_1$ with high probability. We relate this problem to a signal detection problem. We perform this first step in \cref{alg:cr}.
    \vspace{-0.2cm}
    \item Given item $\hat{i}_c$, we identify a feature $j$ such that the gap $\Delta_j$ is large in absolute value, and then classify the items based on this feature. This second step is described in \cref{alg:cbc}. 
    \vspace{-0.3cm}
\end{enumerate}

\subsection{Warm-up: adaptation of Sequential Halving}

As a subroutine of our main method, we introduce a variant of the Sequential Halving (SH) algorithm, introduced in \cite{pmlr-v28-karnin13}.Our version is an adaptation of the Bracketing Sequential Halving described in \citep[Alg.~3]{zhao2023revisiting}. Building on  recent advances in the analysis of SH \cite{zhao2023revisiting} applied to various best arm identification variants \cite{zhao2023revisiting}, we analyze the performance of the method in a specific problem that we introduce. We call this algorithm \texttt{CompareSequentialHalving} (CSH), which is detailed in \cref{alg:CSH}).

Each time we run CSH, we allocate a budget $T$, that the algorithm can fully spend. That is, CSH operates under a fixed budget constraint. 
Following the literature on best arm identification with multiple good arms \cite{berry1997bandit,pmlr-v108-katz-samuels20a,jamieson2016power,de2021bandits}, we incorporate a sub-sampling mechanism. Initially, CSH selects randomly $S_0$, a subset of $L$ entries, where $L$ is a parameter specifying the sub-sampling size, is provided as an input to CSH. Sequential Halving is then applied exclusively to the selected subset, rather than to the entire matrix. The optimal choice for $L$ balances two factors. If $L$ is small, the algorithm concentrates more budget per entry, enabling the detection of smaller gaps. If $L$ is large, we increase the likelihood of including in $S_{0}$ a significant proportion of good entries, ensuring that the quality of remaining entry is not limited by an unlucky draw. We provide an explicit guarantee for \texttt{CSH} in \cref{alg:CSH}, \cref{lem:upperboundCSH}.

We employ CSH as a subroutine in both \cref{alg:cr} and \cref{alg:cbc}. In \cref{alg:cr}, we explore the entire matrix in order to detect an entry that differs from the first row.  In \cref{alg:cbc}, we focus the exploration on a single row, aiming to detect a feature that best separates the two groups.  To achieve this, we introduce a parameter $I$, which represents the subset of items we compare to the first item. The entries in $S_0$ are picked uniformly at random from $I\times [d]$.  In  \cref{alg:cr}, we explore the entire matrix, so $I$ is chosen as $I=\{2,\dots,n\}$. In \cref{alg:cbc}, we explore only one row $i_c$, and we take $I=\{i_c\}$.

\begin{algorithm} 
\caption{\texttt{CompareSequentialHalving (CSH)}}\label{alg:CSH}
\begin{algorithmic}[1]
\REQUIRE $I\subseteq [n]$ subset of experts, $L$ number of halving steps, $T\geq 2^{L+2}$ budget
\ENSURE $(\hat{i},\hat{j})\in[n]\times [d]$
\STATE selects $S_0 \gets \{(i_1,j_1),(i_2,j_2),\dots, (i_{2^L},j_{2^L})\}$ uniformly with replacement from $I\times [d]$ \label{lin:CSHsampling}
\FOR{$l=1,\dots,L$}
\STATE $\tau_l\gets \left\lfloor\frac{T}{2^{L-l+2}L}\right\rfloor$
\FOR{$(i,j)\in S_{l-1}$}
 \STATE draw $\left\{\begin{tabular}{ll}
              $X_{1,j}^{(1)},\dots,X_{1,j}^{(\tau_l)}\sim^{\mathrm{i.i.d.}}\nu_{1,j}$\\
               $X_{i,j}^{(1)},\dots,X_{i,j}^{(\tau_l)}\sim^{\mathrm{i.i.d.}}\nu_{i,j}$ 
        \end{tabular}  \right.$\label{lin:CSHdifferences}
\STATE store $\widehat{D}_{i,j}\gets \frac{1}{\tau_l}\sum_{u=1}^{\tau_l}\left(X_{i,j}^{(u)}-X_{1,j}^{(u)}\right)$
\ENDFOR
\STATE keep in $S_l$ indices $(i,j)\in S_{l-1}$ corresponding to the $2^{L-l}$ largest $\left|\hat D_{i,j}\right|$ 
\ENDFOR
\STATE return $(\hat i,\hat j)\in S_L$
\end{algorithmic}
\end{algorithm}

\subsection{First step: \texttt{CandidateRow}}

As explained in the beginning of this section, we start our procedure by solving a sub-problem which consists on detecting one row of $M$ which differs from the first row. We perform this step in the following \cref{alg:cr}. The guarantees of \cref{alg:cr} are proved in \cref{appendix:B} and are gathered in \cref{thm:cr}. 

\begin{algorithm}
\caption{\texttt{CandidateRow (CR)}} \label{alg:cr}
    \begin{algorithmic}[1]
        \REQUIRE confidence parameter $\delta > 0$
        \ENSURE row index $\hat{i}_{c}\in[n]$
        \STATE initialize $\hat i_{c} \gets 0$, $k\gets 1$
        \WHILE{$\hat i_{c} = 0$} 
        \FOR{$1\leq L\leq L_{\max}$ such that $L\cdot 2^L\leq 2^{k+1}$} \label{line:2.3}
        \STATE $(\hat{i},\hat{j})\gets$\texttt{CSH}($[n]$, $L$, $2^{k+1}$) \label{lin:CSHincr}
        \STATE draw $\left\{\begin{tabular}{ll}
             $X_{1,\hat j}^{(1)},\dots,X_{1,\hat j}^{(2^k)}\overset{\mathrm{i.i.d.}}{\sim}\nu_{1,\hat j}$\\
              $X_{\hat i,\hat j}^{(1)},\dots,X_{\hat i,\hat j}^{(2^k)}\overset{\mathrm{i.i.d.}}{\sim}\nu_{\hat i,\hat j}$
        \end{tabular}  \right.$ \label{line:2.5}
        \IF{$\left|\sum_{t=1}^{2^k}\left(X_{\hat i,\hat j}^{(t)}-X_{1,\hat j}^{(t)}\right)\right|>\sqrt{4\cdot 2^k\log\left((\frac{k^3}{0.15\delta}\right)}$}\label{lin:crterminates}
        \STATE $\hat i_{c}\gets \hat i$
        \ENDIF
        \ENDFOR
        \STATE $k \gets k+1$
        \ENDWHILE
    \end{algorithmic}
\end{algorithm}

In \cref{alg:cr}, we perform multiple runs ot the \texttt{CHS} subroutine, iteratively increasing the budget $T_k=2^{k+1}$ allocated for each run. For a given run of \texttt{CHS}$([n],L,T_k)$, we obtain an entry $(\hat{i},\hat{j})$ (Line~\ref{lin:CSHincr}). In Line~\ref{line:2.5}, we use the same amount of observations $2^{k+1}$ to estimate the gap $|\mu_{\hat i\hat j}-\mu_{1 \hat{j}}|$. Finally, in Line~\ref{line:cbcterminates}, we perform a test based on a Hoeffding's bound to decide  
whether $|\mu_{\hat i\hat j}-\mu_{1 \hat{j}}|>0$ or not. If at some point, this test concludes, \cref{alg:cr} outputs $i_c=\hat{i}$. The threshold chosen in the stopping condition from Line~\ref{lin:crterminates} is designed to assure that with a probability larger than $\delta$, then the selected item $i_c$ belongs to the second group. 

In Line \ref{line:2.3}, we use the definition of $L_{\max}$ as 
\begin{equation*}
   L_{\max} \coloneqq \left\lceil \log_2\left(16dn\log\left(\frac{4\log(8nd)}{\delta}\right)\right)\right\rceil \enspace,
\end{equation*}
which corresponds to the sub-sampling budget required, according to \cref{lem:upperboundCSH}, when $\theta=1/n$ takes the smallest possible value.

From \cref{lem:upperboundCSH}, we know that for $\Delta_{(s)}^2\leq 128 L^2$ and some constant $c$, if the condition 
\begin{align*}
    T_k=2^{k+1} \geqslant c L_{\max}^3 \frac{d\left(\log(1/\delta)+\log\log(nd)\right)}{\theta s \Delta_{(s)}^2}\enspace ,
\end{align*}
holds, then, there exists $L\leqslant L_{\max}$ such that \texttt{CHS}$([n],L,T_k)$ outputs a pair $(i,j)$ such that $\left|\mu_{ij}-\mu_{1j}\right|\geq \left|\Delta_{(s)}\right|/2$. This condition is especially true for the sparsity $s \in [d]$, where the inequality is tightest.

As in \cite{jamieson2016power,saad2023active}, the exponential grid $T_k=2^k$ allows us to adapt the strategy, and reach a budget that scales up to log terms as $O\left(\min_{s\in[d]} \frac{d}{\theta s\Delta_{(s)}^2}\log(1/\delta)\right)$, even without prior knowledge of this quantity by the learner.  

We define for that purpose \begin{equation}
    s^*\in \mathrm{argmax}_{s\in [d]}s\cdot \Delta_{(s)}^2 \enspace.\label{eq:effectivesparsityparameter}
\end{equation}
This quantity, $s^*$, appears as an effective sparsity parameter, as observed in signal detection contexts. Up to a $\log(d)$ factor, the following bound holds
\begin{align}
    \max_{s\in[d]}s\cdot\Delta_{(s)}^2\leq \Vert\Delta\Vert_2^2\leq \log(2d)\max_{s\in [d]}s\cdot \Delta_{(s)}^2\enspace.\label{eq:effectivelysparse}
\end{align}

Finally, we prove that \cref{alg:cr} outputs an item $\hat{i_c}$ which belongs to the second with a probability larger than $1-\delta$, using a budget that is smaller, up to logarithmic terms, than the quantity $\frac{d}{\theta\|\Delta\|_2^2}\log(1/\delta)$. 

\subsection{Second step: \texttt{ClusterByCandidates}}

Assume that, after the first step of our procedure, \cref{alg:cr} provides an item $i_c\in[n]$ such that $\mu_{i_{c}}\neq \mu_1$. Our next goal is to select a feature $j$ such that $|\Delta_j|$ is large enough to allow a quick classification of each of the $n$ items. We propose the procedure \cref{alg:cbc} which shares a similar structure with \cref{alg:cr}. The guarantees of \cref{alg:cbc} are proved in \cref{appendix:B}, in \cref{thm:cbc}. 

We begin by performing several runs of \texttt{CSH}, with $I=\{i_c\}$ in order to detect large entries in the (absolute) gap vector $\Delta$. As in \cref{alg:cr}, we perform \texttt{CSH} with an increasing sequence of budget $T_k=2^{k+1}$, and for each budget $T_k$, we test all possible sub-sampling sizes $L=1,\dots,\tilde L_{\max}$. In Line~\ref{line:cbcforloop}, we define 
\begin{equation*}
   \tilde L_{\max} \coloneqq \left\lceil \log_2\left(16d\log\left(\frac{4\log(8d)}{\delta}\right)\right)\right\rceil \enspace,
\end{equation*}
as the maximum sub-sampling size needed to detect a signal when the sparsity level $s=1$ is the smallest. 

In Line~\ref{line:cbccallCSH}, we call \texttt{CSH}$(\{i_c\},L,2^{k+1})$.  We obtain a feature $j$ and then estimate $|\Delta_j|=|\mu_{i_c j}-\mu_{1 j}|$ in Line~\ref{line:cbcgapestimation}. 
For that, we take $\lfloor 2^{k}/n\rfloor$ samples from $\nu_{i_c,j}$ and $\nu_{1j}$, and compute the sum of differences $\hat{D}=\sum_{t=1}^{\lfloor 2^k/ n\rfloor}\left(X_{i_c,\hat j}-X_{1,\hat j}\right)$. We deduce a high probability lower bound $\widehat{|\Delta|}_j=\frac{1}{\lfloor 2^k/ n\rfloor}(\hat{D}-\epsilon)\leqslant |\Delta_j|$, where $\epsilon$ is defined in Line~\ref{line:cbcepsilon}. Based on $\widehat{|\Delta|}_j$, we can classify (with high probability) each item by sampling the $j$-th feature $O\left(\frac{1}{\widehat{|\Delta|}_j^2}\log(n/\delta)\right)$. We can then assess whether the classification budget required for feature $j$ is feasible given the budget $T_k$. If $T_k\geqslant \frac{cn}{\widehat{|\Delta|}_j^2}\log(n/\delta)$. If it seems that with feature $j$, the estimated classification budget exceeds $T_k$, we discard this feature and repeat \texttt{CSH}, either by increasing $L$ or $T_k$. 

We now bound the budget of our procedure thanks to \cref{lem:upperboundCSH}.  If  $\Delta_{(s)}^2\leq 128 L^2$, if it holds for some constant $c$ that
\begin{align*}
    T_k=2^{k+1} \geqslant c L_{\max}^3 \frac{d\left(\log(1/\delta)+\log\log(d)\right)}{s \Delta_{(s)}^2}\enspace ,
\end{align*}
 then, there exists $L\leqslant \tilde L_{\max}$ such that CHS$({i_c},L,T_k)$ outputs a $j$ such that $\left|\mu_{ij}-\mu_{1j}\right|\geq \left|\Delta_{(s)}\right|/2$. 
 If we also have $\frac{n}{\Delta_{(s)^2}}\log(n/\delta)\leqslant T_k$, then the algorithm would stop with a budget $O(\frac{d}{s\Delta_{(s)}^2\log(1/\delta})$, otherwise it would continue sampling until . 

 Overall, we prove that the total budget of the procedure, up to logarithmic factors, is no more than
 $$\min_{s\in [d]}\left(\frac{d}{s}+n\right)\left(\frac{1}{\Delta_{(s)}^2}+1\right)\log(1/\delta) \enspace.$$

\begin{algorithm}[h]
\caption{\texttt{ClusterByCandidates (CBC) }} \label{alg:cbc}
    \begin{algorithmic}[1]
        \REQUIRE confidence parameter $\delta > 0$, item $i_c\in[n]$
        \ENSURE labels $\hat g \in\{0,1\}^n$ 
        \STATE $\hat g \gets (0,\dots,0)^T\in\{0,1\}^n$, $k\gets \lceil \log_2(n)\rceil$
        \WHILE{True}
        \FOR{$1\leq L\leq \tilde L_{\max}$ such that $L\cdot 2^L\leq 2^{k+1}$} \label{line:cbcforloop}
        \STATE$\hat j \gets \texttt{CSH}(\{i_c\},L,2^{k+1} )$ \label{line:cbccallCSH}
        \STATE draw $\left\{\begin{tabular}{ll}
            $X_{1,\hat j}^{(1)},\dots,X_{1,\hat j}^{(\lfloor 2^k/ n\rfloor)}\sim^{\mathrm{i.i.d.}}\nu_{1,\hat j}$\\
              $X_{i_c,\hat j}^{(1)},\dots,X_{i_c,\hat j}^{(\lfloor 2^k/ n\rfloor)}\sim^{\mathrm{i.i.d.}}\nu_{i_c,\hat j}$
        \end{tabular}  \right.$
        \STATE $\hat D\gets \sum_{t=1}^{\lfloor 2^k/ n\rfloor}\left(X_{i_c,\hat j}-X_{1,\hat j}\right)$\label{line:cbcgapestimation}
        \STATE $\varepsilon\gets\sqrt{4\cdot\lfloor 2^k/ n\rfloor \log(nk^3/0.15\delta)}$ \label{line:cbcepsilon}
        \IF{$\left|\hat D\right|\geq 3\cdot \varepsilon$} \label{line:cbcterminates}
        \FOR{$i=2,\dots,n$}
        \STATE draw $\left\{\begin{tabular}{ll}
            $X_{1,\hat j}^{(1)},\dots,X_{1,\hat j}^{(\lfloor 2^k/ n\rfloor)}\sim^{\mathrm{i.i.d.}}\nu_{1,\hat j}$\\
              $X_{i,\hat j}^{(1)},\dots,X_{i,\hat j}^{(\lfloor 2^k/ n\rfloor)}\sim^{\mathrm{i.i.d.}}\nu_{i,\hat j}$
        \end{tabular}  \right.$
        \STATE $\hat D_i \gets \displaystyle\sum_{t=1}^{\lfloor 2^k/ n\rfloor}\left(X_{i,\hat j}-X_{1,\hat j}\right)$
        \STATE $\hat g(i)\gets \mathbf{1}\left(\left|\hat D_i\right|\geq \varepsilon\right)$ \label{line:cbclabels}
        \ENDFOR
        \STATE \textbf{output} $(\hat{g}(1),\dots,\hat{g}(d))$
        \ENDIF
        \ENDFOR
        \ENDWHILE
    \end{algorithmic}
\end{algorithm}


\subsection{Main Algorithm}

To obtain a complete clustering procedure that is adaptive to $\Delta$ and $\theta$, one simply has to combine Algorithm~\ref{alg:cr} and Algorithm~\ref{alg:cbc}. The overall clustering procedure is then given in \cref{alg:cluster}. 

\begin{algorithm}
\caption{\texttt{BanditClustering }} \label{alg:cluster}
    \begin{algorithmic}[1]
        \REQUIRE confidence parameter $\delta > 0$
        \ENSURE labels $\hat g\in\{0,1\}^n$
        \STATE $i_c\gets \texttt{CR}(\delta/2)$
        \STATE $\hat g\gets \texttt{CBC}(\delta/2, i_c)$
    \end{algorithmic}
\end{algorithm}

\begin{theorem} \label{thm:main}
    For $\delta\in (0,1/e)$, consider Algorithm~\ref{alg:cluster} with entry $\delta$. Define 
    \begin{equation}
        H:=   \frac{d} {\theta}\left(\frac{1}{\Vert \Delta\Vert^2}+\frac{1}{s^*}\right)
        \vphantom{\left[\left(\frac{d}{s}+n\right)\left(\frac{1}{\Delta_{(s)}^2}+1\right)+\frac{\log_+\log\left(1/\Delta_{(s)}^2\right)}{\Delta_{(s)}^2}n\right]}+\min_{s\in [d]}\left(\frac{d}{s}+n\right)\left(\frac{1}{\Delta_{(s)}^2}+1\right)\enspace, \label{def:complexity}
    \end{equation}
    where $s^*$ is the effective sparsity defined in \eqref{eq:effectivelysparse}. 
    
    With a probability of at least $1-\delta$, Algorithm~\ref{alg:cluster} returns $\hat g=g$ with a budget of at most
    \begin{align*}
        T\leqslant \tilde C\cdot \log\left(\frac{1}{\delta}\right)
        \cdot H ,
    \end{align*}
    where there exists a numerical constant $C$, and an index $\tilde{s}=s^* \vee (\lceil d/n\rceil \wedge |\{j\in [d] \; , \Delta_j\ne 0\}|)$, such that $\tilde{C}$ is a logarithmic factor smaller than 
    \begin{align*}
           C &\cdot \left(\log\log(1/\delta)\vee 1\right)^4 \\
           &\cdot\log(dn)^5\log(d)(\log_+\log (1/\Delta_{(\tilde{s})}^2) \vee 1)\enspace.
    \end{align*}
\end{theorem}

We interpret $H$~\eqref{def:complexity} as a non asymptotic sampling complexity, which depends on the instance-specific parameters of our model $\theta$, $\Delta$, $n$, and $d$. The complexity $H$ can be decomposed as two terms:
\begin{itemize}
    \vspace{-0.3cm }
    \item \textbf{First Term: } $\frac{d}{\theta\|\Delta\|^2}\log(1/\delta)$, which correspond to the budget used to identify an item belonging to the second group. The factor $\frac{1}{\theta}$ is necessary, as we need to select $\frac{1}{\theta}$ items to reach an item from the smallest group. Then, for each selected item, we need to decide if its feature vector is equal to $0$ or not. This active test uses a budget $\frac{d}{\|\Delta\|^2}\log(1/\delta)$. 
     \vspace{-0.3cm }
     \item \textbf{Second Term: }  $\min_{s\in [d]}\left(\frac{d}{s}+n\right)\left(\frac{1}{\Delta_{(s)}^2}+1\right)$. This term represents the best trade-off between the exploration of the gap vector $\Delta$ and its exploitation for clustering.  For example, if the gap vector is $s$-sparse takes two values $0$ and $h>0$ (exactly $s$ times), we need $\frac{d}{s}\frac{1}{h^2}\log(1/\delta)$ samples to detect an entry in the support of the gap vector, and then $\frac{1}{h^2}\log(1/\delta)$ samples for the classification of each of the $n$ items.
\end{itemize}

Finally, we mention that the remaining term $d/(\theta s^*)$ in $H$ only dominates in the very specific setting where the non-zero entries of $\Delta$ are really large so that $\|\Delta^2\|\geq s^*$.

The next corollary provides a simplified bound in the specific where the gap vector $\Delta$ only takes two values.
\begin{corollary}\label{cor:simplecase}
    For $\delta\in (0,1/e)$ and $\Delta\in \{0,h\}^d$ with $0<h<1$, with a probability of at least $1-\delta$, Algorithm~\ref{alg:cluster} returns $\hat g=g$ with a budget of at most
    \begin{align*}
        \tilde{C} \cdot \log(1/\delta)\cdot\left(\frac{d}{\theta \Vert \Delta\Vert^2}+\frac{n}{h^2}\right),
    \end{align*}
    where $\tilde{C}$ is a logarithmic factor smaller than 
    \begin{equation*}
           C \cdot \left(\log\log(1/\delta)\vee 1\right)^4 \cdot\log(dn)^5(\log_+\log (1/h^2) \vee 1)\enspace,
    \end{equation*}
      with a numerical constant $C>0$.
\end{corollary}

\section{Lower bounds}\label{section:LB}


In this section, we provide a lower bound on the total number of observations required for any algorithm which is $\delta$-PAC for the clustering task. We provide instance-dependent lower bound, that is  a bound that holds for the specific problem matrix $M$. We establish it by considering a set of environments parametrized by the same matrix $M$ slightly modified. We prove then that an algorithm with a budget too small can not perform well simultaneously on all these environments. 
We assumed in the model that an algorithm solving our problem should adapt to any type of noise, as long as it is $1$-subGaussian. Thus, It is enough for the lower bound to consider Gaussian distributions with variance $1$, for which \cref{assumption:subGaussian} holds.

The lower bound contains two terms, a first term scaling as $\frac{d}{\theta \|\Delta\|^2}\log(1/\delta)$, which can be interpreted as a budget necessary to detect one item from each group, while adapting to the unknown structure of the gap vector $\Delta$. There is also a second term for the lower bound, scaling as $\frac{n}{\Delta_{(1)}^2}\log(1/\delta)$. This bound reflects the fact that, in the case where the most discriminative question is given as an oracle to the learner -- that is $j\in[d]$ such that $|\Delta_j|=|\Delta_{(1)}|$ is maximal -- then, the problem essentially reduces to $n$ Gaussian test of the shape $H_0: X\sim \mathcal{N}(0,1)$ versus $H_1: X\sim \mathcal{N}(\Delta_{(1)},1)$. 

We consider $\mathcal{E}_{per}(M)$ as the set of Gaussian environments constructed from $M$ by permuting rows and columns. Formally, an environment $\tilde{\nu}\in \mathcal{E}_{per}(M)$ is constructed with a permutation $\sigma$ of $[n]$, and a permutation $\tau$ of $[d]$ as follows
\begin{equation} \label{def:env_per}
\tilde\nu_{i,j} = \left\{
    \begin{array}{ll}
       \mathcal{N}(\mu^a_{\tau_j},1)  & \mbox{if  } g(\sigma(i))=a \\
       \mathcal{N}(\mu^b_{\tau_j},1)  & \mbox{if  } g(\sigma(i))=b
    \end{array}
\right. \enspace,
\end{equation}
where $g\in\{a,b\}^n$ denotes the unknown labels associated to matrix $M$. 
 
Intuitively, permuting the rows and columns of $M$ allows us to take into account that, (a) the target labels $g$ is obviously not available to the learner, (b) the structure of the gap vector $\Delta$ is also unknown. 

\begin{theorem} \label{Thm:LB}
Fix $\delta\in(0,1/4)$. Assume that $\mathcal{A}$ is $\delta$-PAC for the clustering task, then, there exists $\tilde\nu\in \mathcal E_{per}(M)$ such that the $(1-\delta)$-quantile of the budget of algorithm $\mathcal{A}$ is bounded as follows
    \begin{equation}
       \mathbb{P}_{\mathcal{A},\tilde\nu}\left(T\geqslant \frac{2(n-2)}{\Delta_{(1)}^2} \log\left(\frac{1}{4.8\delta}\right) \vee \frac{2d}{\theta\|\Delta\|_2^2}\log \frac{1}{6\delta}\right) \geqslant \delta
    \end{equation}
\end{theorem}

First, the term $\frac{2d}{\theta\|\Delta\|^2}\log(1/6\delta)$ matches, up to logarithmic terms, the budget that we employ in the first step to  identify a relevant item. It implies in particular that the first step from \cref{alg:cr} is optimal. 
For the second term, in the case where the gap vector $\Delta$ takes two values, this lower bound matches, up to poly-logarithmic terms, the upper bound from \cref{cor:simplecase}. In summary, when $\Delta$ only takes two values, our budget is optimal with respect to $d$, $n$, $\theta$ and $\log(1/\delta)$. For more general $\Delta$, we conjecture that the trade-off in $H$ in~\eqref{def:complexity} is optimal and unavoidable. 

\section{Experiments}\label{sec:experiments}

We want to underpin our theoretical results by simulation studies with two experiments, where in both cases we sample from normal distribution with variance $1$. In the first experiment, we consider \texttt{CR}, \texttt{CBC} and \texttt{BanditClustering}, and investigate the required budgets in sparse regimes. We compare the performance to a uniform sampling approach for given budgets $T>0$, this is to sample $\tau =\lceil T/nd\rceil$ times from each entry of $M$. We will explain how in this manner we obtain $n$ vectors $\bar X_i\sim\mathcal{N}(\mu_i,\tau^{-1}I_d)$, which corresponds to samples with uncorrelated features concentrated around the mean vectors $\mu^a$ and $\mu^b$, respectively. We will then compare the performance of our algorithms to the performance of the $K$-means algorithm as a popular clustering method for data of this form.

In a second experiment we will look at the \texttt{BanditClustering} algorithm for fixed sparsity but similtanously growing parameters $n$ and $d$, also comparing different choices of the confidence parameter $\delta>0$.

\paragraph{Experiment 1} Consider a small group of $n=20$ items, but a large number of $d =1000$ questions. For $s = \lfloor (d-1)(\gamma-1)/19\rfloor+1$ with $\gamma=1,\dots,20$, define the vector $\Delta^s$ through
\begin{align*}
    \Delta^s_i = \begin{cases}
        15/\sqrt{s}&\text{if }i\leq s\enspace,\\
        0 &\text{else}\enspace .
    \end{cases}
\end{align*}
Note that $\Vert \Delta^s\Vert_2 =15$ for each $s$. We set $\mu^a= 0$ and $\mu^b=\Delta^s$ and construct a matrix $M$ that consists of $10$ rows $\mu^a$ and $10$ rows $\mu^b$, so it is $\theta = 0.5$. We consider $\nu_{i,j}=\mathcal{N}(M_{i,j},1)$. Define $i^*$ as the smallest index $i\in [n]$ such that the $i$-th row of $M$ is different from the first row. For $\delta =0.8$, we run our algorithms $\texttt{CR}(\delta/2)$, $\texttt{CBC}(\delta/2,i^*)$ and $\texttt{BanditCluster}(\delta)$ each $\kappa=5000$ times. The large choice of $\delta$ is due to the fact that our algorithms are quite conservative, such that for this experiment the observed error rate of \texttt{BanditClustering} for the choice $\delta =0.8$ does not differ much from $0.01$ for any $s$. Recall that \texttt{BanditClustering} is a two step procedure, and an error from \texttt{CR} causes in general that \texttt{CBC} will not terminate. For our experiments, we therefore emergency stop \texttt{BanditClustering} if there is an error caused by \texttt{CBC}. We therefore depict the average budgets required by $\texttt{BanditClustering}(\delta)$ for the cases where there was no emergency stop, together with the average budget required by $\texttt{CR}(\delta/2)$ and $\texttt{CBC}(\delta/2,i^*)$ in Figure~\ref{plt:exp1}. The budgets required by $\texttt{BanditClustering}(\delta)$ is depicted together with the $(0.05,0.95)$-quantiles of our simulations. As a benchmark, we compare our algorithm to an approach based on the \texttt{KMeans} method from the Scikit-learn library \cite{scikit-learn}: Given a budget $T$, we sample
$\tau =\lfloor T/nd\rfloor$ observations $X_{i,j}^{(t)}\sim^{\mathrm{i.i.d.}}\mathcal{N}(M_{i,j},1)$ and store $\bar{X}_{i,j}=\frac{1}{\tau}\sum_{t=1}^{\tau}X_{i,j}^{(t)}$. We then cluster by performing the $K$-means algorithm with the vectors $(\bar X_{1j})_{j=1,\dots, d},\dots,(\bar X_{nj})_{j=1,\dots,d}$. We run again $\kappa=5000$ trials of this methods for $10$ values of $T$ between $T_{\min}=170000$ and $T_{\max}=300000$ for the different setups with respect to $s$. In Figure~\ref{plt:exp1}, we illustrate the first point in the time-grid where the respective error rate is below $0.01$.

\begin{figure}[ht]
\vskip 0.2in
\begin{center}
\centerline{\includegraphics[width=0.5\columnwidth]{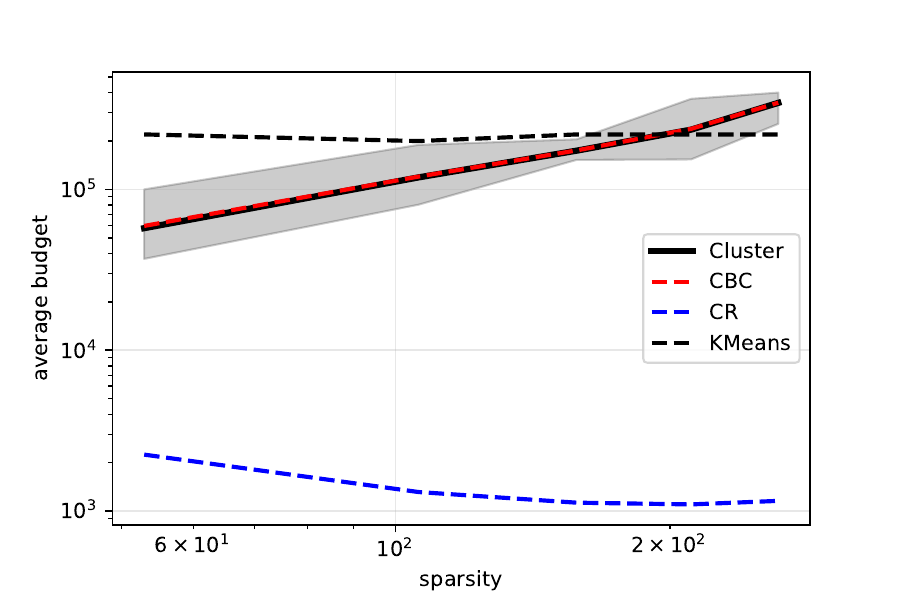}}
\caption{Different budgets for Experiment 1, depending on the sparsity of $\Delta^s$.}\label{plt:exp1}
\end{center}
\vskip -0.2in
\end{figure}

From Figure~\ref{plt:exp1} we can see that in the sparse regime, our algorithm requires less observations than the uniform sampling approach that we chose. Moreover, we see that the required budget of \texttt{CR} seems not to depend much on $s$. On the other hand, for this example, the sample complexity of \texttt{BanditClustering} is clearly driven by \texttt{CBC}, which seems to grow linearly in $s$. These dependencies on $s$ agree with our theory, where the budget of \texttt{CR} in this case is (up to polylogarithmic factors) of order $d/\theta\Vert\Delta\Vert_2^2$ which is constant in $s$, while \texttt{CBC} requires a budget of order $n/(\Delta_i^s)^2\sim n\cdot s$ (up to polylogarithmic factor).

\paragraph{Experiment 2} We consider matrices with dimension $n \in \{100,200,500,1000,2000,5000\}$ and $d = 10\cdot n$. For each $n$ the vector $\tilde \Delta ^{(n)}$ is defined by
\begin{align*}
    \tilde \Delta^{(n)} _j = \begin{cases}
        5 & \text{if } j\leq 10\enspace ,\\
        0& \text{else}\enspace . 
    \end{cases}
\end{align*}
Depending on $n$, we consider $\mu^ a=0$ and $\mu^b =\tilde \Delta^{(n)}$ such as a matrix $M$ consisting of $n/2$ rows $\mu ^a$ and $n/2$ rows $\mu ^b$. Again, we assume $\nu_{i,j}=\mathcal{N}(M_{i,j},1)$. For these setups, we run $\texttt{BC}(\delta)$ with $\delta\in\{0.8,0.5,0.2,0.05\}$ for $\kappa=5000$ trials. In Figure~\ref{plt:exp2}, we depict the average budget required by $\texttt{BC}$ with different parameters, for $\delta=0.05$ we also show the $(0.05,0.95)$-quantiles. As a benchmark, we compare it to $n^2=nd/10$.

\begin{figure}[ht]
\vskip 0.2in
\begin{center}
\centerline{\includegraphics[width=0.5\columnwidth]{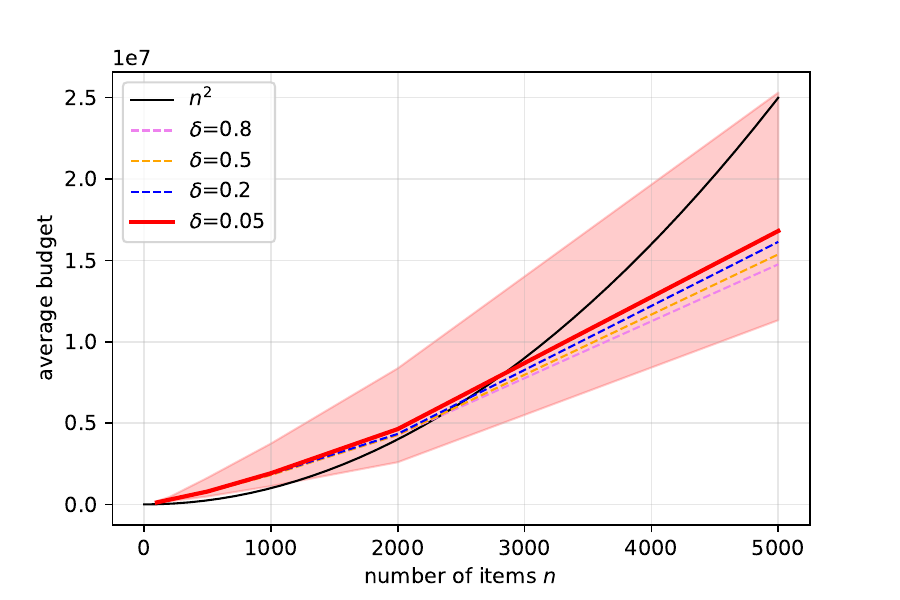}}
\caption{Different budgets for Experiment 2, depending on the dimensionality of the problem $n$ and $d=10\cdot n$.}\label{plt:exp2}
\end{center}
\vskip -0.2in
\end{figure}
Note that if we for example allocate a budget of less than $5n^2=nd/2$ uniformly (at random) among the $nd$ entries of $M$, there will be on average $d/2$ unobserved features for each item $i$. Heuristically, if the sparsity $s$ does not change, but $n$ and $d$ jointly grow, the probability that for some item $i$ we only sample from $\nu_{i,j}$ with $\tilde \Delta^{(n)}_j=0$  goes to one. For such an item, it will be impossible to determine the correct cluster. In other words, clustering based on uniform sampling strategies will fail in this setup, while Figure~\ref{plt:exp2} illustrate that the required budget of our algorithm grows linearly in $n$. Again, this is in line with the bounds from Corollary~\ref{cor:simplecase}, which yields for $d=10\cdot n$ and fixed $\theta$, $s$ and $h$ (up to polylogarithmic factors) a budget of order $n$.

\section{Discussion}\label{sec:discussion}

\paragraph{Comparison to other active clustering settings and batch clustering.} In this work, we consider a bandit clustering setting where the learner can adaptively sample each item-feature pair. This contrasts with~\cite{yang2024optimal,thuot2024active,yavas10838574} where the authors have to sample all the features for each item and cannot focus on most relevant features. Rewriting their results in our setting, the optimal budget for the latter problem is, up to poly-logarithmic terms, of the order of 
\[
\frac{nd\log(1/\delta)}{\|\Delta\|^2}+ \frac{d^{3/2}\sqrt{n\log(1/\delta)}}{\|\Delta\|^2}\ . 
\]
Comparing this with our main result (\cref{thm:main}), we first observe that the ability to adaptively select features allows to remove the so-called high-dimensional terms $d^{3/2}\sqrt{n\log^{1/2}(1/\delta)}/\|\Delta\|^2$ that occurs when the number of features is large - $d\geq n\log(1/\delta)$. Second, the adaptive queries allow to drastically decrease the budget in situation where the vector $\Delta$ contains a few large entries so that a few feature are especially relevant to discriminate. To illustrate this, consider e.g. a setting as in Corollary~\ref{cor:simplecase} where $\Delta\in \{0,h\}^d$ takes $s$ non-zero values where the partition is balanced so that $\theta=1/2$. Then, our budget is of the order of 
\[
\log(1/\delta)\left[\frac{d}{sh^2}+ \frac{n}{h^2}\right]\ , 
\]
which represents a potential reduction by a factor $n\wedge \frac{d}{s}$ compared to~\cite{yang2024optimal,thuot2024active}.
\vspace{-0.3cm}
\paragraph{Extension to a larger number of groups.}. Throughout this work, we assumed that the items are clustered into $K=2$ groups. We could straightforwardly extend our methodology to a larger $K>2$ by first identifying $K$ representative items, one per group, and then looking at significant features to discriminate between the groups.
A key challenge for achieving optimality in this setting is determining whether the algorithm should focus on all $K(K-1)/2$ pairwise discriminative features or a smaller, more informative subset.

\vspace{-0.3cm}
\paragraph{Extension to heterogeneous groups.} We also assumed in this work that all the items within a group are similar, meaning their corresponding mean vectors $\mu_i$ are exactly the same. This assumption could be weakened by allowing the $\mu_i$'s within a group to be close but not necessarily equal. If  prior knowledge of intra-group heterogeneity is available, it is not difficult to adapt our procedure to this new setting by increasing some tuning parameters. Without such prior knowledge, further research is needed to tune the procedure in a data-adaptive way.

\section*{Acknowledgements}
The work of V. Thuot and N. Verzelen has been partially supported by grant ANR-21-CE23-0035 (ASCAI,ANR). The work of M. Graf has been partially supported by the DFG Forschungsgruppe FOR 5381 "Mathematical Statistics in the Information Age - Statistical Efficiency and Computational Tractability", Project TP 02, and by the DFG on the French-German PRCI ANR ASCAI CA 1488/4-1 "Aktive und Batch-Segmentierung, Clustering und Seriation: Grundlagen der KI".



\bibliography{biblio.bib}
\bibliographystyle{icml2025}

\newpage
\appendix
\onecolumn

\section{Notation}

To ease the reading, we gather the main notation below
\begin{itemize}\vspace{-0.4cm}
    \item $n$ number of items, $d$ number of features
    \item $\mu^a\neq \mu^b\in \mathbb{R}^{d}$ feature vectors of the two groups
    \item $\mu_i\in\{\mu^a,\mu^b\}$, $i\in [n]$ feature vector of item $i$
    \item $M$ matrix with rows $(\mu_i)_i$ (with size $n\times d$)
    \item $g\in\{0,1\}^n$ true labels (fixing $g(1)=0$)
    \item $X\sim \nu_{i,j}$ for $(i,j)\in[n]\times [d]$: $\mathbb{E}[X]=\mu_{ij}$ with $X-\mu_{ij}$ $1$-sub-Gaussian
    \item $\delta\in(0,1)$ prescribed probability of error
    \item$\theta \coloneqq \frac{\sum_{i=1}^n\mathbf{1}(g(i)=0)\wedge \sum_{i=1}^n\mathbf{1}(g(i)=1)}{n}$ balancedness 
    \item $\Delta \coloneqq \mu^a-\mu^b\ne 0$ gap vector
    \item $s^*\in \mathrm{argmax}_{s\in [d]}s\cdot \Delta_{(s)}^2$ effective sparcity
\end{itemize}
Moreover, as we repeatedly compare the entries of $M$ with the first row, we introduce in the proofs the quantities
\begin{itemize}
    \item $D_{i,j}\coloneqq M_{i,j}-M_{1,j} \; \mbox{, for } (i,j)\in[n]\times[d].$
\end{itemize}
    
\section{Analysis of Algorithm~\ref{alg:CSH}}\label{appendix:A}

We analyze here the performance of \texttt{CSH}. 

\begin{lemma}\label{lem:upperboundCSH}
Consider $\delta\in(0,1)$, $s\in[d]$ and $h>0$ such that $\left|\Delta_{(s)}\right|\geq h$. Consider $I\subset [n]$ and define the relative proportion of items in the second group as $\alpha=\frac{|\{i\in I;\; g(i)=1\}|}{|I|}\enspace.$
Consider Algorithm~\ref{alg:CSH}-- \texttt{CSH}$(I,L,T)$ --with input $I,L,T$ such that 
\begin{align}
    & L =\left\lceil \log_2\left(16\frac{d}{\alpha s}\log\left(\frac{4\log(8|I|d)}{\delta}\right)\right)\right\rceil \enspace, \label{eq:CSH_L}\\
    & T \geq  516\frac{L^3\cdot 2^{L}}{h^2} \vee   2^{L+1}L \label{eq:CSH_T}\enspace.
\end{align}
Then \texttt{CSH}$(I,L,T)$ outputs a pair $(\hat i,\hat j)$ such that $\left| \mu_{\hat i, \hat j}-\mu_{1 \hat j}\right|\geq h/2$ with probability $\geq 1-\delta$.
\end{lemma}

Remark that for $I=[n]$, then $\alpha\geqslant \theta$. If $I=\{i_c\}$ where $g(i_c)\ne g(1)$, then $\alpha=1$. 

Throughout this section, we will prove Lemma~\ref{lem:upperboundCSH} with $I=[n]$. The general result directly follows from the case where $I$ contains all items. To see that, we just have to see that $\alpha$ is equal to the balancedness of the matrix $M|_I$ restricted to the rows in $I$, and we would replace $n$ by $|I|$, and $\theta$ by $\alpha$. 

Therefore consider Algorithm~\ref{alg:CSH} with input $I=[n]$, $L=\left\lceil \log_2\left(16\frac{d}{\theta s}\log\left(\frac{4\log(8nd)}{\delta}\right)\right)\right\rceil$ and $T=516\frac{L^3\cdot 2^{L}}{h^2}\vee 2^{L+1}L$. 

For the following proofs, define $\gamma \coloneqq h/2L$ and   
\begin{align*}
    U_l\coloneqq \{(i,j)\in S_l:\ |D_{i,j}|\geq h-l\gamma\},\quad l=0,1,\dots,L\enspace .
\end{align*}
Lemma~\ref{lem:upperboundCSH} is a direct consequence of following statement:
\begin{lemma}\label{lem:proportionlowerbounded}
    With probability of at least $1-\delta$, it holds
    \begin{equation*}
        \frac{|U_l|}{|S_l|}\geq 2^{-L+3} \log\left(\frac{4\log(8nd)}{\delta}\right)\quad \forall \ l=0,1,\dots,L\enspace .
    \end{equation*}
\end{lemma}
\begin{proof}[Proof of Lemma~\ref{lem:upperboundCSH}]
	The first statement follows by Lemma~\ref{lem:proportionlowerbounded}, since the $(\hat{i},\hat{j})$ that is returned lies in $S_L$, which contains only one pair of indices. Since we have that $U_L\subseteq S_L$ is nonempty with probability at least $1-\delta$, this implies that $(\hat i,\hat j)\in U_L$ as claimed.

	For the second part, we can simply replace $[n]$ by $\{i_c\}$, $n$ by $1$ and $\theta$ by $1$ and adapt the proof of Lemma~\ref{lem:proportionlowerbounded}.
\end{proof}
\begin{proof}[Proof of Lemma~\ref{lem:proportionlowerbounded}]
We will prove via induction over $l$, that
    \begin{equation*}
        \frac{|U_k|}{|S_k|}\geq 2^{-L+3} \log\left(\frac{4\log(8nd)}{\delta}\right)\quad \forall \ k=0,1,\dots,l\enspace 
    \end{equation*}
    holds with probability at least 
    \begin{align*}
        1-(l+1)\left(\frac{\delta}{4\log(8nd)}\right)^2\enspace .
    \end{align*}
    The statement follows then from
    \begin{align*}
        (L+1)\cdot\left(\frac{\delta}{4\log(8nd)}\right)^2 & 
        \leq\left(2\log\left(16\frac{d}{\theta s}\log\left(\frac{4\log(8nd)}{\delta}\right)\right)+1\right)\cdot \left(\frac{\delta}{4\log(8nd)}\right)^2\\
       & \leq 3\log\left(8\frac{d}{\theta s}\right)\cdot \left(\frac{\delta}{4\log(8nd)}\right)^2 +2\log\log\left(\left(\frac{4\log(8nd)}{\delta}\right)^2\right)\cdot \left(\frac{\delta}{4\log(8nd)}\right)^2\\
        & \leq \frac{3}{4}\frac{\delta^2}{\log(8nd)}+\frac{\delta}{4\log(8nd)}\leq \delta\enspace ,
    \end{align*}
    where we used that $\lceil\log_2(x)\rceil\leq 2\log(x)$ for $x>5$, and the last line is obtained by $8d/\theta s\leq 8nd$ and $2x\cdot\log\log (1/x)\leq \sqrt{x}$ for $x\in(0,1)$. 

\paragraph{The base case $l=0$} Recall from line~\ref{lin:CSHsampling} in Algorithm~\ref{alg:CSH}, we sample $2^L$ indices from $[n]\times[d]$. From the entries in $M$, there are at least $\theta n\cdot s$ entries $\left|D_{i,j}\right|\geq h$. So the random variables 
\begin{equation*}
    X_{t}^{(0)}\coloneqq \mathbf{1}\left(\left| D_{i_t,j_t}\right|\geq h\right),\quad t=1,\dots,2^L
\end{equation*}
are i.i.d. Bernoulli random variables with $\mathbb{P}(X_t^{(0)}=1)\geq \theta\frac{s}{d}$. Consequently we have
    \begin{equation*}
        \mu^{(0)}\coloneqq \mathbb{E}\left[\sum_{t=0}^{2^L}X_t^{(0)}\right]=2^L\theta \frac{s}{d}\geq 16\log\left(\frac{4\log(8nd)}{\delta}\right)\enspace .
    \end{equation*}
    From the second inequality in Lemma~\ref{lem:cernov}, we obtain that 
    \begin{align*}
        \mathbb{P}\left(\sum_{t=0}^{2^L}X_t^{(0)}\leq 8\log\left(\frac{4\log(8nd)}{\delta}\right)\right)\leq\mathbb{P}\left(\sum_{t=0}^{2^L}X_t^{(0)}\leq \frac{1}{2}\mu^{(0)}\right)
        \leq\exp\left(-\frac{\mu^{(0)}}{8}\right) \leq \left(\frac{\delta}{4\log(8nd)}\right)^2\enspace.
    \end{align*}
    So we have 
    \begin{equation*}
        |U_0|=\sum_{t=0}^{2^L}X_t^{(0)}>8\log\left(\frac{4\log(8nd)}{\delta}\right)=|S_0|2^{-L+3}\log\left(\frac{4\log(8nd)}{\delta}\right)
    \end{equation*}
    with probability at least $1-\left(\frac{\delta}{4\log(8nd)}\right)^2$.

    \paragraph{Induction step: from $l$ to $l+1$} Consider the event $\xi_l$, defined as
        \begin{equation*}
        \frac{|U_k|}{|S_k|}\geq 2^{-L+3} \log\left(\frac{4\log(8nd)}{\delta}\right)\quad \forall \ k=0,1,\dots,l\enspace .
    \end{equation*}
    We want to show
    \begin{align*}
        \mathbb{P}\left(\xi_{l}\right)\geq 1-(l+1)\left(\frac{\delta}{4\log(8nd)}\right)^2\quad \Rightarrow \quad \mathbb{P}\left(\xi_{l}\right)\geq 1-(l+2)\left(\frac{\delta}{4\log(8nd)}\right)^2
    \end{align*}
    Note that $\xi_{l+1}\subseteq \xi_l$, so showing
    \begin{align*}
        \mathbb{P}\left(\xi_{l+1}\mid\xi_l\right)\geq 1-\left(\frac{\delta}{4\log(8nd)}\right)^2
    \end{align*}
    suffices to conclude
    \begin{align*}    \mathbb{P}\left(\xi_{l+1}\right)&=\mathbb{P}\left(\xi_{l}\right)\cdot \mathbb{P}\left(\xi_{l+1}\mid \xi_l\right)
    \geq \left(1-(l+1)\left(\frac{\delta}{4\log(8nd)}\right)^2\right)\left(1-\left(\frac{\delta}{4\log(8nd)}\right)^2\right)
    \\
    &\geq 1-(l+2)\left(\frac{\delta}{4\log(8nd)}\right)^2\enspace . 
    \end{align*}

    When we condition on the event $\xi_l$, this implies the condition 
    \begin{equation*}
    |U_l|\geq 2^{-L+3}|S_l|\log\left(\frac{4\log(8nd)}{\delta}\right)\enspace .
    \end{equation*}
    Recall that in line~\ref{lin:CSHdifferences} of Algorithm~\ref{alg:CSH}, we first sample
    \begin{align*}
        X_{1,j}^{(1)},\dots,X_{1,j}^{(\tau_{l+1})}\sim^{\mathrm{i.i.d.}}\nu_{1,j}\enspace ,
        X_{i,j}^{(1)},\dots,X_{i,j}^{(\tau_{l+1})}\sim^{\mathrm{i.i.d.}}\nu_{i,j}
    \end{align*}
    for each $(i,j)\in S_l$ and store 
    \begin{equation*}
        \hat D_{i,j}=\frac{1}{\tau_{l+1}}\sum_{u=1}^{\tau_{l+1}}X_{i,j}^{(u)}-X_{1,j}^{(u)}\enspace .
    \end{equation*}
    Since we assumed that $X_{i,j}^{(u)}-M_{i,j}\in SG(1)$ and $X_{1,j}^{(u)}-M_{i,j}\in SG(1)$, this implies 
    \begin{align*}
        \sum_{u=1}^{\tau_{l+1}}\left(X_{i,j}^{(u)}-X_{1,j}^{(u)}-D_{i,j}\right)\in SG(2\tau_{l+1})
    \end{align*}
    and we obtain
    \begin{align}
        \mathbb{P}\left( \hat D_{i,j}-D_{i,j}\geq \gamma/2\right)=\left(\sum_{u=1}^{\tau_{l+1}}\left(X_{i,j}^{(u)}-X_{1,j}^{(u)}-D_{i,j}\right)\geq \tau_{l+1}\gamma/2\right)
        \leq \exp\left(-\frac{\tau_{l+1}\gamma^2}{16}\right) \eqqcolon p_l\enspace ,\label{eq:concentrationforcernov}
    \end{align}
    and likewise 
        \begin{align*}
        \mathbb{P}\left( D_{i,j}-\hat D_{i,j}\geq \gamma/2\right)\leq p_l\enspace ,
    \end{align*}
    
    For $(i,j)\in U_l$, this implies that 
    \begin{align*}
        \mathbb{P}\left(\left|\hat D_{i,j}\right|\leq h-(l+1/2)\gamma\right)\leq p_l\enspace .
    \end{align*}
    So we can construct i.i.d. Bernoulli random variables $B_{i,j}$ with 
    \begin{equation*}
    \mathbb{P}(B_{i,j}=1)=p_l    
    \end{equation*}
    and 
    \begin{equation*}
    \left|\hat D_{i,j}\right|\leq h-(l+1/2)\gamma\quad \Rightarrow \quad B_{i,j}=1    
    \end{equation*}
    for $(i,j)\in U_l$.
    By Lemma~\ref{lem:cernov} it follows (by letting $\kappa=\frac{1}{2p_l}-1$) 
	\begin{align*}
		\mathbb{P}\left(\sum_{(i,j)\in U_l}B_{i,j}\geq |U_l|/2\mid | U_l|=\eta\right)&\leq \exp\left(\kappa p_l\eta-(1+\kappa)p_ln\log(1+\kappa)\right)\\
		&\leq \exp\left(\eta\left(1/2-p_l-\frac{\log(1/p_l)-\log(2)}{2}\right)\right)\\
		&\leq \exp\left(\eta\left(1/2-\exp(-\tau_{l+1}\gamma^2/16)-\frac{\tau_{l+1}\gamma^2/16-\log(2)}{2}\right)\right)\\
		&\leq \exp(-2^{l-2}\eta)\enspace .
	\end{align*}
	The last inequality follows, since we assumed 
    \begin{equation*}
        T\geq 516\cdot 2^L \cdot L^3 /h^2\vee 2^{L+1}L \geq 128\cdot 2^L\cdot L/\gamma^2\vee 2^{L+1}L\enspace ,
    \end{equation*}
    such that 
    \begin{equation*}
        \tau_{l+1}=\left\lfloor \frac{T}{2^{L-l+1}L}\right\rfloor \geq\left\lfloor 32\frac{2^{l+1}}{\gamma^2}\vee 2^{l-1}\right\rfloor\geq 16\frac{2^{l+1}}{\gamma^2}
    \end{equation*}
    (by $\lfloor x\rfloor \geq x/2$ for $x\geq 1$), from where we can conclude
	\begin{align*}
		1/2-\exp(-\tau_{l+1}\gamma^2/16)-\frac{\tau_{l+1}\gamma^2/16-\log(2)}{2}\leq -2^{l-2}
	\end{align*}
	for $l\geq 0$. For 
    \begin{equation*}
        \eta\geq 2^{-L+3}|S_l|\log\left(\frac{4\log(8nd)}{\delta}\right)=2^{-l+2}\log\left(\left(\frac{4\log(8nd)}{\delta}\right)^2\right)\enspace ,
    \end{equation*}
    this implies
	\begin{align*}
		\mathbb{P}\left(\sum_{(i,j)\in U_l}B_{i,j}\geq |U_l|/2\mid | U_l|=\eta\right)
		\leq \left(\frac{\delta}{4\log(8nd)}\right)^2
	\end{align*}
    and therefore 
    \begin{align}
        \mathbb{P}\left(\sum_{(i,j)\in U_l}B_{i,j}\geq |U_l|/2\mid \xi_l\right)
		\leq \left(\frac{\delta}{4\log(8nd)}\right)^2\label{eq:proportiongoodarms}
    \end{align}

    Next, define 
    \begin{equation*}
        V_l\coloneqq \{(i,j)\in S_l:\ |D_{i,j}|<h-(l+1)\gamma \}\enspace. 
    \end{equation*}
    Note that $S_{l+1}\setminus U_{l+1}\subseteq V_l$. So if 
    \begin{equation*}
        |V_l|< 2^{-L+3}|S_l|\log\left(\frac{4\log(8nd)}{\delta}\right)\enspace ,
    \end{equation*}
    this implies
    \begin{align}
        |U_{l+1}|&=|S_{l+1}|-|S_{l+1}\setminus U_{l+1}|
        \notag\\ 
        &\geq |S_{l+1}|-|V_l|\notag \\
        &> |S_{l+1}|-2^{-L+3}|S_l|\log\left(\frac{4\log(8nd)}{\delta}\right)\notag\\
        &=\left(1-2^{-L+2}\right)|S_{l+1}|\log\left(\frac{4\log(8nd)}{\delta}\right)\notag\\
        &\geq 2^{-L+3}|S_{l+1}|\log\left(\frac{4\log(8nd)}{\delta}\right)\enspace ,\label{eq:proportionbadarmssmall}
    \end{align}
    since we have $L\geq \log_2(16)=4$. Therefore, consider the nontrivial case
        \begin{equation*}
        |V_l|\geq  2^{-L+3}|S_l|\log\left(\frac{4\log(8nd)}{\delta}\right)\enspace ,
    \end{equation*}
    Like before, we have from \eqref{eq:concentrationforcernov} that 
        \begin{align*}
        \mathbb{P}\left(\left|\hat D_{i,j}\right|\geq h-(l+1/2)\gamma\right)\leq p_l\enspace 
    \end{align*}
    for $(i,j)\in V_l$. Note that we can again define Bernoulli random variables $C_{i,j}$ with
    \begin{equation*}
        \mathbb{P}(C_{i,j}=1)=p_l
    \end{equation*}
    and 
    \begin{equation*}
        \left|D_{i,j}\right|\geq h-(l+1/2)\gamma \quad \Rightarrow \quad C_{i,j} =1
    \end{equation*}
    for all $(i,j)\in V_l$. We can again show that conditional $|V_l|=\eta$ with 
        \begin{equation*}
        \eta\geq 2^{-L+3}|S_l|\log\left(\frac{4\log(8nd)}{\delta}\right)\enspace ,
    \end{equation*}
    it holds 
    \begin{equation}
        \sum_{(i,j)\in V_l}C_{i,j}\geq \eta/2 \label{eq:proportionbadarms}
    \end{equation}
    with probability $1-\left(\frac{\delta}{4\log(8nd)}\right)^2$.

    Now if $\overline\Delta$ is the median of the $\hat D_{i,j}$, $(i,j)\in S_l$, it is either $\overline \Delta < h-(l+1/2)\gamma$ or $\overline \Delta \geq h-(l+1/2)\gamma$. In the case $\overline \Delta < h-(l+1/2)\gamma$, the bound \eqref{eq:proportiongoodarms} tells us that $U_{l+1}$ contains at least half of the indices of $U_l$, in other words,
    \begin{equation*}
        \frac{|U_{l+1}|}{|S_{l+1}|}\geq \frac{|U_l|}{2|S_{l+1}|}=\frac{|U_l|}{|S_l|}\enspace ,
    \end{equation*}
    with probability at least $1-\delta$. In the case $\overline \Delta \geq h-(l+1/2)\gamma$, we either directly conclude the induction step from \eqref{eq:proportionbadarmssmall}, or we know from \eqref{eq:proportionbadarms} that the number of $(i,j)\in S_{l+1}$ with $|D_{i,j}|\leq h-(l+1)\gamma$ is less than half the number of arms in $V_l$, and in particular,
    \begin{equation*}
        \frac{|U_{l+1}|}{|S_{l+1}|}\geq 1-\frac{|V_l|}{2|S_{l+1}|}=1-\frac{|V_l|}{|S_l|}=\frac{|U_l|}{|S_l|}\enspace ,
    \end{equation*}
    with probability at least $1-\left(\frac{\delta}{4\log(8nd)}\right)^2$. Combining both cases yields the claim.
    \end{proof}

\section{Analysis of Algorithm~\ref{alg:cr}}\label{appendix:B}

Using the results of Lemma~\ref{lem:upperboundCSH}, we can now determine theoretical guarantees for Algorithm~\ref{alg:cr}, \texttt{CR}($\delta$).

We present the individual guarantees offered by \cref{alg:cr} in the following proposition. 

\begin{proposition}\label{thm:cr}
    Let $\delta\in(0,1/e)$. Then, with probability larger than $1-\delta$, Algorithm~\ref{alg:cr}-- $\texttt{CR}(\delta)$ --returns an index $\hat i_{c}$, such that it holds that $\mu_{\hat i_c}\neq \mu_1$, and moreover the total budget is upper bounded by  
    \begin{align*}
        \tilde{C} \cdot \log\left(\frac{1}{\delta}\right)\cdot \frac{d} {\theta}\left(\frac{1}{\Vert \Delta\Vert^2}+\frac{1}{s^*}\right) \enspace,
    \end{align*}
    where $\tilde{C}$ is a logarithmic factor smaller than 
    \begin{equation*}
           C \cdot \left(\log\log(1/\delta)\vee 1\right)^4 \cdot\log(dn)^5\log(d)(\log_+\log (1/\Delta_{(s^*)}^2) \vee 1)\enspace,
    \end{equation*}
      with a numerical constant $C>0$ and $\log_+(x)\coloneqq \log(x\vee 1)$ for $x\in\mathbb{R}$.
\end{proposition}

\begin{proof}[Proof of Proposition~\ref{thm:cr}]
    Consider $s\in[d]$ and $k\geq 1$ minimal, such that for
    \begin{equation*}
        L =\left\lceil \log_2\left(16\frac{d}{\theta s}\log\left(\frac{16\log(8nd)}{\delta}\right)\right)\right\rceil
    \end{equation*}
    it holds
    \begin{equation}
        516\frac{L^32^L}{\Delta_{(s)}^2}\vee 2^{L+1}L\leq 2^k \label{eq:crstop1}
    \end{equation}
    and 
    \begin{equation}
        3714\cdot \frac{\log(1/\delta)+\log_+\log\left(1/\Delta_{(s)}^2\right)}{\Delta_{(s)}^2}\vee 2<  2^{k}\enspace .\label{eq:crstop2}
    \end{equation}
        The proof consists of three parts:
    \begin{enumerate}
        \item\label{item:cr1} \texttt{CSH}$([n],L,2^k)$ returns $(\hat{i},\hat{j})$ with $\left| D_{\hat{i},\hat{j}}\right|\geq \left|\Delta_{(s)}\right|/2$, with probability at least $1-\delta/2,$
        \item \label{item:cr2}for $\left| D_{\hat{i},\hat{j}}\right|\geq \left|\Delta_{(s)}\right|/2$, if we sample 
    \begin{equation*}
        X_{1,\hat j}^{(1)},\dots,X_{1,\hat j}^{(2^k)}\sim^{\mathrm{i.i.d.}}\nu_{1,\hat j} \quad \mathrm{and}\quad  X_{\hat i,\hat j}^{(1)},\dots,X_{\hat i,\hat j}^{(2^k)}\sim^{\mathrm{i.i.d.}}\nu_{\hat i,\hat j}\enspace ,
    \end{equation*}
    it holds with probability of at least $1-\frac{\delta}{0.3k^3}$ that 
    \begin{equation*}
        \left|\frac{1}{2^k}\sum_{t=1}^{2^k}X_{\hat i,\hat j}^{(t)}-X_{1,\hat j}^{(t)}\right|> \sqrt{\frac{4}{2^k}\log\left(\frac{k^3}{0.15\delta}\right)}\enspace ,
    \end{equation*}
    \item \label{item:cr3} for any $k'\geq 1$, if $D_{i,j}=0$ and we sample
        \begin{equation*}
        X_{1,j}^{(1)},\dots,X_{1, j}^{(2^{k'})}\sim^{\mathrm{i.i.d.}}\nu_{1, j} \quad \mathrm{and}\quad  X_{i,j}^{(1)},\dots,X_{ i, j}^{(2^{k'})}\sim^{\mathrm{i.i.d.}}\nu_{ i, j}\enspace ,
    \end{equation*}
        it holds with probability of at least $1-\frac{\delta}{0.3{k'}^3}$ that 
    \begin{equation*}
        \left|\frac{1}{2^{k'}}\sum_{t=1}^{2^{k'}}X_{ i, j}^{(t)}-X_{1, j}^{(t)}\right|\leq \sqrt{\frac{4}{2^{k'}}\log\left(\frac{{k'}^3}{0.15\delta}\right)}\enspace .
    \end{equation*} 
    \end{enumerate}
    
    From point \ref{item:cr1} and \ref{item:cr2} one can conclude that Algorithm~\ref{alg:cr} terminates in the $L^{\mathrm{th}}$ step of the $k^{\mathrm{th}}$ iteration at the latest with probability at least $1-\delta/2-\delta/0.3k^3$. If it has not terminated before, by point \ref{item:cr1}, we obtain in line \ref{lin:CSHincr} $(\hat i,\hat j)$ with $\left| D_{\hat i,\hat j}\right|\geq \left|\Delta_{(s)}\right|/2$ and by point \ref{item:cr2}, that for such $(\hat i,\hat j)$, the algorithm terminates in line \ref{lin:crterminates}, returning $\hat i$. 

    If the algorithm terminates for some $k'< k$ or in the $k^{\mathrm{th}}$ round, but for some other $L'$, by line \ref{lin:crterminates} this means that the algorithm returns some $\hat i$ such that for some $\hat j$ it holds 
        \begin{equation*}
        \left|\frac{1}{2^{k'}}\sum_{t=1}^{2^{k'}}X_{ i, j}^{(t)}-X_{1, j}^{(t)}\right|> \sqrt{\frac{4}{2^{k'}}\log\left(\frac{{k'}^3}{0.15\delta}\right)}\enspace .
    \end{equation*} 
    So point \ref{item:cr3} implies for each iteration $k'$ and each $L'$ we iterate over, that we do not return an index $\hat i$ with $\mu_{\hat i}=\mu_1$, with probability at least $1-\delta/0.3k^3$.

So by the union bound, Algorithm~\ref{alg:cr} returns $\hat i$ with $\mu_{\hat i}\neq \mu_1$ with probability at least
\begin{align*}
    1-\delta/2-\sum_{k'\leq k }\quad \sum_{\substack{1\leq L\leq L_{\max}\\  L2^L\leq 2^{k'}}}\delta/(0.3{k'}^3)&\geq1-\delta/2-\delta\cdot 0.3\sum_{k\geq 1}\frac{1}{k^2}\geq 1-\delta/2-0.3\frac{\pi^2}{6}\delta > 1-\delta\enspace .
\end{align*}

So we are left with proving the three points.

\paragraph{Proof of \ref{item:cr1}} By Lemma~\ref{lem:upperboundCSH} and inequality~\eqref{eq:crstop1}, calling \texttt{CSH}$([n],L,2^k)$ in line~\ref{lin:CSHincr} of Algorithm~\ref{alg:cr} yields a pair $(\hat i,\hat j)$ with $| D_{\hat i,\hat j}|\geq |\Delta_{(s)}|/2$ with probability at least $1-\delta/2$.

\paragraph{Proof of \ref{item:cr2}} Note that for
    \begin{equation}
        \hat D_{\hat i\hat j}\coloneqq \frac{1}{2^k}\sum_{t=1}^{2^k}X_{\hat i,\hat j}^{(t)}-X_{1,\hat j}^{(t)}\enspace ,  \label{eq:defdeltaij}
    \end{equation}
    by an application of Hoeffding's inequality we know
    \begin{equation}
        \hat  D_{\hat i,\hat j}\in\left[ D_{\hat i,\hat j}-\sqrt{\frac{4}{2^k}\log\left(\frac{k^3}{0.15\delta}\right)}, D_{\hat i,\hat j}+\sqrt{\frac{4}{2^k}\log\left(\frac{k^3}{0.15\delta}\right)}\right]\label{eq:confidencinterval}
    \end{equation}
    with a probability of at least $1-0.3\delta/k^3$. Note that from inequality~\eqref{eq:crstop2} we know by the monotonicity of $\log\log(x)/x$ for $x\geq e^2$ that
    \begin{align*}
        \frac{16}{2^k}\log\left(\frac{k^3}{0.15\delta}\right)=\frac{16}{2^k}\left(\log(1/\delta)+3\log\log(2^k)+3\log\left(\frac{1}{\log(2)}\right)+\log(20/3)\right)\leq 80\frac{\log(1/\delta)+\log\log\left(2^k\right)}{2^k}\enspace .
    \end{align*}
    
    We want to prove the bound
    \begin{align}
        \frac{16}{2^k}\log\left(\frac{k^3}{0.15\delta}\right)\leq \Delta_{(s)}^2/4\enspace . \label{eq:stoppingcriterion}
    \end{align}
    Let us first consider the case $\Delta_{(s)}^2\geq 1/e$. We can bound
    \begin{align*}
        80\frac{\log(1/\delta)+\log\log(2^k)}{2^k}&\leq80\frac{\log(1/\delta)+\log(2^k)}{2^k}\\
        &= 80\frac{\log(1/\delta)+\log(2^k\Delta_{(s)}^2)-\log(\Delta_{(s)}^2)}{2^k\Delta_{(s)}^2}\Delta_{(s)}^2 \\
        &\leq 80\frac{2\log(1/\delta)+\log\log(2^k\Delta_{(s)}^2)}{2^k\Delta_{(s)}^2}\enspace .
    \end{align*}
    From inequality~\eqref{eq:crstop1}, we know
    \begin{align*}
        2^{k}\Delta_{(s)}^2\geq 3714 \log(1/\delta)\enspace ,
    \end{align*}
    and we can therefore use that $x\mapsto \log(x)/x$ is decreasing for $x\geq e$ to obtain
    \begin{align*}
        80\frac{\log(1/\delta)+\log\log(2^k)}{2^k}&\leq 80\frac{2\log(1/\delta)+\log(3714\log(1/\delta))}{3714\log(1/\delta)}\Delta_{(s)}^2\\
        &\leq 80\frac{3+\log(3714)}{3714}\Delta_{(s)}^2\leq \Delta_{(s)}^2/4\enspace .
    \end{align*}

    Next, consider the case $\Delta_{(s)}^2\leq 1/e$. Then we know from inequality\eqref{eq:crstop2} that 
    \begin{align*}
        2^k\geq 3714\frac{\log(1/\delta)+\log\log(1/\Delta_{(s)}^2)}{\Delta_{(s)}^2}\enspace .
    \end{align*}
    Because $x\mapsto \log\log(x)/x$ is decreasing for $x\geq e^2$, we can bound 
    \begin{align*}
        80\frac{\log(1/\delta)+\log\log(2^k)}{2^k}
        \leq 80\frac{\log(1/\delta)+\log\log\left(3714\frac{\log(1/\delta)+\log\log(1/\Delta_{(s)}^2)}{\Delta_{(s)}^2}\right)}{3714\left(\log(1/\delta)+\log\log\left(1/\Delta_{(s)}^2\right)\right)}\Delta_{(s)}^2\enspace .
    \end{align*}
    For $a,b\geq e$ it holds $\log\log(ab)\leq \log(2)+\log\log(a)+\log\log(b)$, so we can bound
    \begin{align*}
        &\log\log\left(3714\frac{\log(1/\delta)+\log\log(1/\Delta_{(s)}^2)}{\Delta_{(s)}^2}\right)\\
        \leq &\log(2)+\log\log(1/\Delta_{(s)}^2)+\log\log\left(3714\left(\log(1/\delta)+\log\log\left(1/\Delta_{(s)}^2\right)\right)\right)\\
        \leq& \log(2\cdot 3714)+2\log\log\left( 1/\Delta_{(s)}^2\right)+\log(1/\delta)\enspace .
    \end{align*}
    This allows us to bound
    \begin{align*}
        80\frac{\log(1/\delta)+\log\log(2^k)}{2^k}
        &\leq 80 \frac{(2+\log(2\cdot 3714))\log(1/\delta)+3\log\log\left(1/\Delta_{(s)}^2\right)}{3714\left(\log(1/\delta)+\log\log\left(1/\Delta_{(s)}^2\right)\right)}\Delta_{(s)^2}\\
        &\leq \frac{80(2+\log(2\cdot 3714))}{3714}\Delta_{(s)}^2\leq \Delta_{(s)}^2/4\enspace .
    \end{align*}

    For $\left| D_{\hat i,\hat j}\right|\geq \left|\Delta_{(s)}\right|/2$, we have with high probability according to \eqref{eq:confidencinterval} that 
    \begin{align*}
    \left|\hat D_{\hat i,\hat j}\right|\geq \sqrt{\frac{4}{2^k}\log\left(\frac{k^3}{0.15\delta}\right)}\enspace .
    \end{align*}
    
    \paragraph{Proof of \ref{item:cr3}} Analogously to \eqref{eq:defdeltaij} and \eqref{eq:confidencinterval} we can use Hoeffding's inequality to show that for
    \begin{equation*}
        \hat D_{i,j}\coloneqq \frac{1}{2^{k'}}\sum_{t=1}^{2^{k'}}X_{ij}^{(t)}-X_{1j}^{(t)}
    \end{equation*}
    it holds 
    \begin{equation*}
        \left|\hat D_{i,j}\right|\leq \sqrt{\frac{4}{2^{k'}}\log\left(\frac{{k'}^3}{0.15\delta}\right)}
    \end{equation*}
    with probability at least $1-\frac{\delta}{0.3 {k'}^3}$.

\paragraph{Bounding the budget:} 
First, we can bound
\begin{align*}
    L_{\max}&\leq2\log\left( 16 nd\log\left(\frac{16\log(8nd)}{\delta}\right)\right) \\
    &\leq 10 \log\left( nd\log\left(\frac{16\log(8nd)}{\delta}\right)\right)\\
    &\leq 10\left(\log(nd)+\log\log\left(\frac{16\log(8nd)}{\delta}\right)\right)\\
    &\leq 10\left(\log(nd)+\log\left((nd)^4+\log(1/\delta)\right)\right)\\
    &\leq 70\left(\log(nd)+\log\log(1/\delta)\right)\enspace .
\end{align*}
At the same time, we have that
\begin{align*}
    2^L &\leq 32 \frac{d}{\theta s}\log\left(\frac{16 \log(8nd)}{\delta}\right)\\
    &\leq 32\frac{d}{\theta s}\left(\log\log(nd)+\log(64/\delta)\right)\\
    &\leq 192\frac{d}{\theta s}\left(\log\log(nd)+\log(1/\delta)\right)\enspace .
\end{align*}
So, if we define $C\coloneqq156\cdot70^3\cdot  192$ and $k^*$ the minimal $k\geq 1$ such that
\begin{align*}
    2^{k+1}\geq C\min_{s\in[d]}&\left(\frac{(\log(nd)+\log\log(1/\delta))^3(\log\log(nd)+\log(1/\delta))d}{\theta s \Delta_{(s)}^2}\right.\\&\qquad
    \left.+\frac{d(\log(nd)+\log\log(1/\delta))(\log\log(nd)+\log(1/\delta))}{\theta s}+\frac{d}{\theta s}\frac{\log(1/\delta)+\log_+\log(1/\Delta_{(s)}^2)}{\Delta_{(s)}^2}+
    \right) \enspace ,
\end{align*}
we can see that by \eqref{eq:crstop1} and \eqref{eq:crstop2} Algorithm~\ref{alg:cr} terminates and returns $\hat{i}$ such that $\mu_1\neq \mu_{\hat{i}}$ with a probability of at least $1-\delta$. Moreover, on this event of high probability, the algorithm terminates after at most
\begin{align*}
    &\sum_{k=1}^{k^*}\sum_{1\leq L\leq L_{\max}:\ L\cdot 2^L\leq 2^{k+1}}2\cdot 2^{k+1}\leq8 L_{\max}2^{k^*}\\
    \leq C'\cdot L_{\max}\min_{s\in[d]}&\left(\frac{(\log(nd)+\log\log(1/\delta))^3(\log\log(nd)+\log(1/\delta))d}{\theta s \Delta_{(s)}^2}\right.\\&\qquad
    \left.+\frac{d(\log(nd)+\log\log(1/\delta))(\log\log(nd)+\log(1/\delta))}{\theta s}+\frac{d}{\theta s}\frac{\log(1/\delta)+\log_+\log(1/\Delta_{(s)}^2)}{\Delta_{(s)}^2}+
    \right) \enspace ,
\end{align*} 
by minimality of $k^*$, where $C'>0$ is some numerical constant that might change. To obtain the claimed upper bound, note that by \eqref{eq:effectivelysparse} we know $1/s^*\Delta_{(s^*)}^2\leq \log(2d)/\Vert \Delta\Vert_2^2$ and therefore
\begin{align*}       C'L_{\max}\min_{s\in[d]}&\left(\frac{(\log(nd)+\log\log(1/\delta))^3(\log\log(nd)+\log(1/\delta))d}{\theta s \Delta_{(s)}^2}\right.\\&\qquad
    \left.+\frac{d(\log(nd)+\log\log(1/\delta))(\log\log(nd)+\log(1/\delta))}{\theta s}+\frac{d}{\theta s}\frac{\log(1/\delta)+\log_+\log(1/\Delta_{(s)}^2)}{\Delta_{(s)}^2}+
    \right)\\
    &\leq C'' (\log(nd)+\log\log(1/\delta))^4(\log(nd)+\log(1/\delta)(\log_+\log(1/\Delta_{s^*}^2)\vee 1)\frac{d}{\theta }\left(\frac{1}{\Vert \Delta \Vert^2}+\frac{1}{s^*}\right)\enspace ,
\end{align*}
where $C''>0$ is some numerical constant. Reassembling the logarithmic terms yields the claim.
\end{proof}
    
\section{Analysis of Algorithm~\ref{alg:cbc}}\label{apendix:C}

Now, we prove the correctness and we upper bound the budget of \cref{alg:cbc}. 
\begin{proposition}\label{thm:cbc}
    Let $\delta\in (0,1/e)$, let $i_c\in[n]$ such that $\mu_{i_c}\neq \mu_1$. Then Algorithm~\ref{alg:cbc}-- \texttt{CBC}$(\delta,i_c)$ --returns $\hat g = g$, with probability at least $1-\delta$, with a budget of at most
    \begin{align*}
\tilde C\cdot \log(1/\delta)\cdot  \min_{s\in[d]} &\left[\left(\frac{d}{s}+n\right)\left(\frac{1}{\Delta_{(s)}^2}+1\right)\right] \enspace  ,
    \end{align*}
 where $\tilde{C}$ is a logarithmic factor smaller than 
    \begin{equation*}
           C \cdot \left(\log\log(1/\delta)\vee 1\right)^4 \cdot\log(d)^5 \cdot \log_+\log\left(1/\Delta_{(\tilde{s})}^2\right)\enspace,
    \end{equation*}
      with a numerical constant $C>0$, and $\tilde{s}=\lceil d/n\rceil \wedge |\{j\in [d] \; , \Delta_j\ne 0\}|$. 
\end{proposition}

The proof of Proposition~\ref{thm:cbc} does not differ much from the proof of Proposition~\ref{thm:cr}: Again, we have to bound the time where \texttt{CSH} returns an index pair for which the stopping condition is fulfilled with high probability. The main difference is, that we also need a guarantee for correct clustering using these indices, which also leads to a change of the stopping rule.
\begin{proof}
	Consider $s\in [d]$ and $k\in\mathbb{N}$, $k> \log_2(n)$ minimal, such that for 
	\begin{align*}
		L=\left\lceil \log_2\left(16\frac{d}{s}\log\left(\frac{16 \log(8d)}{\delta}\right)\right)\right\rceil 
	\end{align*}
	it holds
	\begin{align}
		516\frac{L^32^L}{\Delta_{(s)}^2}\vee 2^{L+1}L\leq 2^k  \label{eq:cbcstop1}
	\end{align}
	and 
	\begin{equation}
    34423\cdot\frac{\left(\log(1/\delta)+\log_+ \log\left(1/\Delta_{(s)}^2\right)+\log n \right)\cdot n}{\Delta_{(s)}^2}\vee 2n\leq 2^{k }\enspace . \label{eq:cbcstop2}
	\end{equation}
    The proof relies on the two following facts:
    \begin{enumerate}
        \item\label{item:cbc1} \texttt{CSH}$([n],L,2^k)$ returns $( i_,\hat j)$ with $\left| D_{ i_c,\hat j}\right|\geq \left|\Delta_{(s)}\right|/2$, with probability at least $1-\delta/2$,
        \item\label{item:cbc2} we have that jointly for all iterations $k'\geq 1$ and $1\leq L\leq \tilde L_{\max}$ with $2^LL\leq 2^{k+1}$, for some $j\in[d]$ (chosen each time in line \ref{line:cbccallCSH}) and all $1<i\leq n$, when we draw
        \begin{equation*}
            X_{1,\hat j}^{(1)},\dots,X_{1,\hat j}^{(\lfloor 2^{k'}/n\rfloor)}\sim^{\mathrm{i.i.d.}}\nu_{1,\hat j} \quad \mathrm{and}\quad X_{i,\hat j}^{(1)},\dots,X_{i,\hat j}^{(\lfloor 2^{k'}/n\rfloor)}\sim^{\mathrm{i.i.d.}}\nu_{i,\hat j}
        \end{equation*}
        holds
        \begin{equation*}
            \left|\sum_{t=1}^{\lfloor 2^{k'}/n\rfloor}\left(X_{i,\hat j}^{(t)}-X_{1,\hat j}^{(t)}- D_{i,\hat j}\right)\right|\leq \sqrt{4\cdot \lfloor 2^{k'}/n\rfloor\log\left(\frac{n{k'}^3}{0.15\delta}\right)} \enspace ,
        \end{equation*}
        uniformly with probability at least $1-\delta/2$. 
    \end{enumerate}
    Point \ref{item:cbc1} is a direct consequence of Lemma~\ref{lem:upperboundCSH} and \eqref{eq:cbcstop1}. Point \ref{item:cbc2} follows directly from Hoeffding's inequality and a union bound over all $k\geq 1$, $L\leq L_{\max}$ such that $L2^L\leq 2^{k+1}$ and $i\in[n]$. Indeed, each inequality for itself holds with probability at least $1-0.3\delta/n{k'}^3$, so the intersection must hold with a probability of at least 

\begin{align*}
    1-\sum_{k'\geq 1}\sum_{\substack{1\leq L \leq L_{\max}\\ L2^L\leq 2^{k'+1}}} \sum_{i=2}^n0.3\delta /n{k'}^3\geq 1-0.3\delta\sum_{k'\geq 1}\frac{1}{{k'}^2}\geq 1-\delta/2\enspace .
\end{align*}
    
    We will prove that Alogorithm~\ref{alg:cbc} terminates at the latest in the $L^{\mathrm{th}}$ round of the $k^{\mathrm{th}}$ iteration and clusters correctly with probability at least $1-\delta$, namely on the intersection of the high probability events of point \ref{item:cbc1} and \ref{item:cbc2} which we will call $\xi_{\mathrm{cbc}}$.

    \paragraph{Algorithm~\ref{alg:cbc} terminates at the latest in the $k^{\mathrm{th}}$ iteration} Assume we are on $\xi_{\mathrm{cbc}}$. By point \ref{item:cbc1}, we know that at round $L$ of iteration $k$ it holds $\left| D_{i_c,\hat j}\right|\geq \left|\Delta_{(s)}\right|/2$ for $\hat j$ obtained in line~\ref{line:cbccallCSH}. We want to prove 
    \begin{align}
        \frac{64}{\left\lfloor \frac{2^k}{n}\right\rfloor}\log\left(\frac{nk^3}{0.15\delta}\right)\leq \Delta_{(s)}^2/4\enspace. \label{eq:concetrationwhencbstops}
    \end{align}
    
    Note that by \eqref{eq:cbcstop2}, it holds
    \begin{align*}
        \frac{64}{\left\lfloor \frac{2^k}{n}\right\rfloor}\log\left(\frac{nk^3}{0.15\delta}\right)\leq \frac{128n}{2^k}\left(\log n + 3\log \log2^k+\log( 20/3) +\log(1/\delta)\right) 
        \leq 640 n\frac{\log(1/\delta)+\log n+\log\log 2^k}{2^k}\enspace .
    \end{align*}

    Again, consider first the case $\Delta_{(s)}^2\geq 1/e$. In this case, we know from \eqref{eq:cbcstop2} that
    \begin{align*}
        34423\left(\log(1/\delta)+\log n \right) \cdot n \leq 2^k \Delta_{(s)}^2\enspace .
    \end{align*}
    We can use that $x\mapsto \log(x)/x$ is decreasing for $x\geq e$ and obtain 
    \begin{align*}
        640 n\frac{\log(1/\delta)+\log n+\log\log(2^k)}{2^k}&\leq640 n\frac{\log(1/\delta)+\log n+\log(2^k)}{2^k}\\
        &=640 n \frac{\log(1/\delta) + \log n + \log (2^k\Delta_{(s)}^2)-\log(\Delta_{(s)}^2)}{2^k\Delta_{(s)}^2}\Delta_{(s)}^2\\
        &\leq 640 n\frac{2\log(1/\delta)+\log n+\log(2^k\Delta_{(s)}^2)}{2^k\Delta_{(s)}^2}\Delta_{(s)}^2\\
        & \leq \frac{640}{34423}\cdot \frac{2\log(1/\delta)+\log n+\log \left(34423 n(\log(1/\delta)+\log n\right)}{\log(1/\delta)+\log n}\Delta_{(s)}^2\\
        & \leq \frac{640}{34423}\cdot \frac{2\log(1/\delta)+2\log n+\log 34423+\log(\log(1/\delta)+\log n)}{\log(1/\delta)+\log n}\Delta_{(s)}^2\\
        & \leq \frac{640\cdot (3+\log( 34423))}{34423}\Delta_{(s)}^2\leq \Delta_{(s)}^2/4\enspace.
    \end{align*}
    This proves \eqref{eq:concetrationwhencbstops} in the case $\Delta_{(s)}^2 \geq 1/e$.

    Consider $\Delta_{(s)}^2 \geq 1/e$. Then, by \eqref{eq:cbcstop2}, we know 
    \begin{align*}
        2^k\geq 34423n\frac{\log(1/\delta)+\log n+\log\log\left(1/\Delta_{(s)}^2\right)}{\Delta_{(s)}^2}\enspace .
    \end{align*}
    We can apply that $x\mapsto \log\log (x) /x$ is decreasing for $x\geq e^2$ and obtain
    \begin{align*}
        640 n \frac{\log(1/\delta)+\log n +\log \log 2^k}{2^k}
        \leq \frac{640}{34423}\frac{\log(1/\delta)+\log n+\log \log\left(34423 n\frac{\log(1/\delta)+\log n+\log\log\left(1/\Delta_{(s)}^2\right)}{\Delta_{(s)}^2}\right)}{\log(1/\delta)+\log n+\log\log\left(1/\Delta_{(s)}^2\right)}\Delta_{(s)}^2\enspace .
    \end{align*}
    Note, that
    \begin{align*}
        &\log \log\left(34423 n\frac{\log(1/\delta)+\log n+\log\log\left(1/\Delta_{(s)}^2\right)}{\Delta_{(s)}^2}\right)\\
        \leq &\log(2)+\log\log\left(1/\Delta_{(s)}^2\right)+\log\log\left(34423 n\left(\log(1/\delta)+\log n +\log\log\left(1/\Delta_{(s)}^2\right)\right)\right)\\
        \leq &\log(2\cdot 34423)+\log\log\left(1/\Delta_{(s)}^2\right)+\log( n) +\log\left(\log(1/\delta)+\log n +\log\log\left(1/\Delta_{(s)}^2\right)\right)\\
        \leq&\left(\log(2\cdot 34423)+1\right)\log(1/\delta)+2\log\log\left(1/\Delta_{(s)}^2\right)+2\log( n)\enspace ,
    \end{align*}
    where we used $\log \log(a\cdot b)\leq \log(2)+\log\log( a)+\log\log(b)$ for $a,b\geq e$. Thus, it holds
    \begin{align*}
        640 n\frac{\log(1/\delta)+\log n+\log \log(2^k)}{2^k}
        &\leq \frac{640}{34423}\frac{(2+\log(2\cdot 34423))\log(1/\delta)+2\log n+2\log\log\left(1/\Delta_{(s)}^2\right)}{\log(1/\delta)+\log n+\log\log (1/\Delta_{(s)}^2)}\Delta_{(s)}^2\\
        &\leq \frac{640(2+\log(2\cdot 34423))}{34423}\Delta_{(s)}^2\leq \Delta_{(s)}^2/4\enspace ,
    \end{align*}
    which proves \eqref{eq:concetrationwhencbstops}.

    Inequality \eqref{eq:concetrationwhencbstops} implies
    \begin{align*}
        \left| D_{i_c\hat j}\right|\geq \left|\Delta_{(s)}\right|/2\geq 4\cdot\sqrt{\frac{4}{\left\lfloor\frac{2^k}{n} \right\rfloor}\log\left(\frac{nk^3}{0.15\delta}\right)}
    \end{align*}
    and by points \ref{item:cbc1} and \ref{item:cbc2} we have a guarantee that
    \begin{align*}
        \left|\hat  D_{i_c\hat j}\right|\geq 3\cdot\sqrt{\frac{4}{\left\lfloor\frac{2^k}{n} \right\rfloor}\log\left(\frac{nk^3}{0.15\delta}\right)}\enspace .
    \end{align*}
    By line~\ref{line:cbcterminates} of Algorithm~\ref{alg:cbc}, this is sufficient for the algorithm to terminate after the $L^{\mathrm{th}}$ round of iteration $k$.
    
\paragraph{Algorithm~\ref{alg:cbc} clusters correctly} Consider the first $k'\in{\mathbb{N}}$ with $k'> \log_2(n)$ such that for the samples
\begin{align*}
    X_{1,\hat j}^{(1)},\dots,X_{1,\hat j}^{(\lfloor 2^{k'}/ n\rfloor)}\sim^{\mathrm{i.i.d.}}\nu_{1,\hat j} \quad\mathrm{and}\quad X_{i_c,\hat j}^{(1)},\dots,X_{i_c,\hat j}^{(\lfloor 2^{k'}/ n\rfloor)}\sim^{\mathrm{i.i.d.}}\nu_{i_c,\hat j}
\end{align*}
we have that 
\begin{align*}
    \frac{1}{\left\lfloor 2^{k'}/n\right\rfloor}\left|\sum_{t=1}^{\left\lfloor 2^{k'}/n\right\rfloor}X_{i_c,\hat j}^{(t)}-X_{1,\hat j}^{(t)}\right|> 3\cdot\sqrt{\frac{4}{\left\lfloor\frac{2^{k'}}{n} \right\rfloor}\log\left(\frac{n{k'}^3}{0.15\delta}\right)}\enspace .
\end{align*}
Then by line~\ref{line:cbcterminates}, we know that after completing the iteration Algorithm~\ref{alg:cbc} terminates. From point \ref{item:cbc2} we know that on $\xi_{\mathrm{cbc}}$ it holds
\begin{align*}
    \left| D_{i_c,\hat j}\right|> 2\cdot\sqrt{\frac{4}{\left\lfloor\frac{2^{k'}}{n} \right\rfloor}\log\left(\frac{n{k'}^3}{0.15\delta}\right)}\enspace .
\end{align*}
So if for each $i\geq 2$ we sample again
\begin{align*}
    X_{1,\hat j}^{(1)},\dots,X_{1,\hat j}^{(\lfloor 2^{k'}/ n\rfloor)}\sim^{\mathrm{i.i.d.}}\nu_{1,\hat j} \quad\mathrm{and}\quad X_{i,\hat j}^{(1)},\dots,X_{i,\hat j}^{(\lfloor 2^{k'}/ n\rfloor)}\sim^{\mathrm{i.i.d.}}\nu_{i,\hat j} \enspace ,
\end{align*}
then for the averages holds again by point \ref{item:cbc2} that 
\begin{align*}
    \frac{1}{\left\lfloor 2^{k'}/n\right\rfloor}\left|\sum_{t=1}^{\left\lfloor 2^{k'}/n\right\rfloor}X_{i,\hat j}^{(t)}-X_{1,\hat j}^{(t)}\right| >\sqrt{\frac{4}{\left\lfloor\frac{2^{k'}}{n} \right\rfloor}\log\left(\frac{n{k'}^3}{0.15\delta}\right)}
\end{align*}
if and only if $D_{i,\hat j}\neq 0$. So on $\xi_{\mathrm{cbc}}$, the labeling in line \ref{line:cbclabels} yields to a perfect clustering $\hat{g}=g$.

\paragraph{Bounding the budget:}

Similar to the proof of Theorem~\ref{thm:cr}, we can bound
\begin{equation*}
     \tilde L_{\max}\leq 70(\log(d)+\log\log(1/\delta))\enspace .
\end{equation*}
and
\begin{equation*}
    2^L\leq 192\frac{d}{s}(\log \log d+\log(1/\delta))\enspace .
\end{equation*}
So again by defining $C\coloneqq 156\cdot 70^3\cdot 192$ and letting $k^*$ being minimal such that
\begin{align*}
    2^{k+1}\geq C\min_{s\in[d]}&\left(\frac{( \log d+\log\log(1/\delta))^3(\log\log d+\log(1/\delta))d}{s\Delta_{(s)}^2}\right.\\
    &\qquad +\frac{(\log d+\log\log(1/\delta))(\log\log d+\log(1/\delta))d}{s}\\
    &\qquad +\left. \frac{\left(\log(1/\delta)+\log_+\log\left(1/\Delta_{(s)}^2\right)+\log n\right)\cdot n}{\Delta_{(s)}^2}\right)
\end{align*}
we know from \eqref{eq:cbcstop1} and \eqref{eq:cbcstop2} that with probability at least $1-\delta$, Algorithm~\ref{alg:cbc} terminates and clusters correctly, spending a budget of at most
\begin{align*}
    \sum_{k=1}^{k^*}\sum_{1\leq L\leq \tilde L_{\max}:\ L\cdot 2^L\leq 2^{k+1}}2\cdot 2^{k+1}\leq 8 \tilde L_{\max}2^{k^*}\\
    \leq C'  (\log d+\log\log(1/\delta))\min_{s\in[d]}\left[\left(\frac{( \log d+\log\log(1/\delta))^3(\log\log d+\log(1/\delta))}{\Delta_{(s)}^2}+1\right)\frac{d}{s}\right.\\
    \qquad +\left.\left( \frac{\log(1/\delta)+\log_+\log\left(1/\Delta_{(s)}^2\right)+\log n}{\Delta_{(s)}^2}+1\right)n\right]\enspace ,
\end{align*}
where $C>0$ is a numerical constant. Inserting $\tilde s$ in the right hand side and gathering the logarithmic terms like in the proof of Proposition~\ref{thm:cr} yields the claim.
\end{proof}

\section{Proof of the lower bounds}\label{appendix:D}

The lower bound in \cref{Thm:LB} consists of two terms, which we prove separately. In the proofs, we use $T_{i,j}$ as the number of time  a procedure selects the pair $(i,j)\in[n]\times [d]$. 

\begin{lemma} \label{lemma:LB1}
    The $(1-\delta)$-quantile of the budget of any $\delta$-PAC algorithm $\mathcal{A}$ is bounded as follows
    \begin{equation}
        \max_{\tilde\nu \in \mathcal E_{per}(M)} \mathbb{P}_{\tilde\nu,\mathcal{A}}\left(T\geqslant \frac{2d}{\theta\|\Delta\|_2^2}\log \frac{1}{6\delta}\right) \geqslant \delta \enspace.
    \end{equation}
\end{lemma}

\begin{proof}[Proof of \cref{lemma:LB1}]

Fix an algorithm $\mathcal{A}$, and let $\mathcal{E}_{per}(M)$ denote the set of Gaussian environment constructed from $M$ by permutation. 
For the sake of the proof, we introduce $\mathbb P_{\sigma,\tau}$ as the probability distribution induced by the interaction between algorithm $\mathcal{A}$, and the environment defined in \eqref{def:env_per}  with $\sigma$, $\tau$. We permute the rows of $M$ because obviously, the label vector $g$ is not available to the learner. Moreover, we permute the columns of $M$ in order to take into account that the structure of the gap vector $\Delta$ is unknown, in particular the information of the feature with the largest gap. 

Without loss of generality, we assume that $\mu_a=\mathbf{0}$ and $\mu_b=\Delta$, with the group of items associated with the feature vector $\Delta$ being the smallest group.

Define $\chi$ as the smallest integer such that for any permutations $\sigma$ and $\tau$ of $\{1, \dots, n\}$ and $\{1, \dots, d\}$, the following holds
\begin{equation} \label{eq:LB1a}
    \mathbb{P}_{\sigma,\tau}(T > \chi) \leqslant \delta \enspace.
\end{equation}

The goal of the proof is to upper bound $\chi$. Intuitively, we show that for small $\chi$, there exists a permutation of $M$ for which it is impossible to detect a nonzero entry.

Introduce $\mathbb{P}_0$ as the probability distribution induced by $\mathcal{A}$, where $X_t \sim \mathcal{N}(0, 1)$. This corresponds to an environment in which all items belong to the same group. We will prove that in this pathological environment, the algorithm $\mathcal{A}$ will use a budget larger than $\chi$ with probability at least $1 - 2\delta$.

To this end, consider $\nu(g, \mu)$ as the environment constructed with partition $g$ and two groups of items with feature vectors $\mathbf{0}$ and $\mu$, where $\mu$ tends to zero. Since $\mathcal{A}$ is $\delta$-PAC, there exist distinct partitions $g \neq g'$ and an event $A$ such that
\begin{equation*}
    \mathbb{P}_{\nu(g,\mu)}(A,T\leqslant \chi)+\mathbb{P}_{\nu(g',\mu)}(A^c,T\leqslant \chi)\leqslant 2\delta \enspace.
\end{equation*}

For example, choosing $g(1) = 0$, $g(2) = 1$, and $g'(1) = 0$, $g'(2) = 0$, the event $\{\hat{g}(1) = \hat{g}(2)\}$ suffices. Conditionally on $T \leqslant \chi$, $\mathbb{P}_{\nu(g, \mu)}$ and $\mathbb{P}_{\nu(g', \mu)}$ converge in total variation to $\mathbb{P}_0$ as $\mu \to 0$. Thus, it follows that:
\begin{equation} \label{eq:LB1b}
    \mathbb{P}_0(T \leqslant \chi) \leqslant 2\delta \enspace.
\end{equation}

Using \cref{eq:LB1a}, \cref{eq:LB1b}, and the Bretagnolle-Huber inequality \citep[see][Thm.~14.2]{lattimore2020bandit}, we obtain
\begin{equation*}
  \frac{1}{2}\exp\left(-\KL(\Pr_0,\Pr_{\sigma,\tau})\right) \leqslant  \mathbb{P}_0(T\leqslant \chi)
+\mathbb{P}_{\sigma,\tau}(T>\chi) \leqslant   3\delta  \enspace.
\end{equation*}
Thus,
 \begin{equation} \label{eq:LB1c}
 \log\frac{1}{6\delta} \leqslant \KL(\Pr_0,\Pr_{\sigma,\tau})  \enspace.
 \end{equation}

Using the divergence decomposition for $\KL$ \citep[see][Lemma.~15.1]{lattimore2020bandit}, and the Gaussian assumption on the model, we have 
 \begin{equation} \label{eq:LB1d}
    \KL(\mathbb{P}_0,\mathbb{P}_{\sigma,\tau}) 
    = \displaystyle\sum_{i,j}\mathbb{E}_0[T_{i,j}]\KL(\mathbb{P}_0^{i,j},\mathbb{P}_{\sigma,\tau}^{i,j}) 
    =   \displaystyle\sum_{i,j}\mathbb{E}_0[T_{i,j}]\mathds{1}_{\sigma(i)\in S} \frac{\Delta^2_{\tau(j)}}{2} \enspace.
 \end{equation}
 
 Averaging over all permutations $\sigma, \tau$, and using \cref{eq:LB1c,eq:LB1d}, we have:
 \begin{align}
        \log \frac{1}{6\delta} \leqslant \frac{1}{n!}\frac{1}{d!}\displaystyle \sum_{\sigma,\tau} \mathbb{E}_0[T_{i,j}]\mathds{1}_{g(\sigma(i))=1} \frac{\Delta^2_{\tau(j)}}{2} \enspace.
 \end{align}
 
Observe that each element in  $i\in\{1,\dots,n\}$ (resp. $j\in \{1, \dots, d\}$) appears exactly $(n-1)!$ (resp. $(d-1)!$) times in the multi-set $\{\sigma(i)\}_{\sigma}$ (resp. $\{\tau(j)\}_{\tau}$), so that 
 \begin{align*}
        \frac{1}{n!}\frac{1}{d!} \displaystyle \sum_{\sigma,\tau} \sum_{i,j} \mathbb{E}_0[T_{i,j}]\mathds{1}_{g(\sigma(i))=1} \frac{\Delta^2_{\tau(j)}}{2} & = \frac{(n-1)!}{n!}\frac{(d-1)!}{d!} \displaystyle \sum_{k,l} \sum_{i,j} \mathbb{E}_0[T_{i,j}]\mathds{1}_{g(k)=1} \frac{\Delta^2_{l}}{2} \\
     & = \frac{1}{n}\sum_{k\in[n]}\mathds{1}_{g(k)=1} \frac{\|\Delta\|_2^2}{2d} \mathbb{E}_0[T]\enspace.
 \end{align*}

 Since the group associated with $\Delta$ is the smallest, $\frac{1}{n} \sum_{k \in [n]} \mathds{1}_{g(k) = 1} = \theta$. Using a modified algorithm $\mathcal{A'}$ that stops at $T \wedge \chi$, we can bound $\mathbb{E}_0[T] \leqslant \chi$. Finally, it follows that:
\begin{equation*}
    \chi \geqslant \frac{2d}{\theta \|\Delta\|_2^2} \log \frac{1}{6\delta} \enspace.
\end{equation*}

Since $\chi$ is the maximum over all permuted environments constructed with $M$ of the $(1-\delta)$-quantile of the budget, this inequality concludes the proof of \cref{lemma:LB1}.
\end{proof}

\begin{lemma} \label{lemma:LB2}
Assume that $\delta<1/2$. If $\mathcal{A}$ is $\delta$-PAC for the clustering problem, then for any environment $\nu$,
\begin{equation} \label{eq:lower_bound_2}
    \mathbb{E}_{\mathcal{A},\nu}[T] \geqslant \frac{2(n-2)}{\Delta_{(1)}^2} \log\left(\frac{1}{2.4\delta}\right) \enspace,
\end{equation}
where $|\Delta_{(1)}| = \max_{j \in [d]} |\Delta_j|$.
\end{lemma}

\begin{proof}[Proof of \cref{Thm:LB}]
Observe that \cref{lemma:LB2} does not directly provide a high-probability lower bound on the budget. We now prove how this expectation bound induces a bound on the $(1-\delta)$-quantile of the budget.

Let $\mathcal{A}$ be any $\delta$-PAC algorithm. Assume, by contradiction, that 
\[
\mathbb{P}_{\nu,\mathcal{A}}\left(T \geqslant \frac{2(n-2)}{\Delta_{(1)}^2} \log\left(\frac{1}{4.8\delta}\right)\right) < \delta.
\]
We modify $\mathcal{A}$ such that it stops at time 
\[
T' := T \wedge \frac{2(n-2)}{\Delta_{(1)}^2} \log\left(\frac{1}{4.8\delta}\right) \enspace.
\]
If $\mathcal{A}$ reaches time $\frac{2(n-2)}{\Delta_{(1)}^2} \log\left(\frac{1}{4.8\delta}\right)$, it stops sampling and outputs an error. The resulting algorithm $\mathcal{A}'$ is $2\delta$-PAC, with a budget satisfying 
\[
\mathbb{E}_{\mathcal{A'},\nu}[T'] \leqslant \frac{2(n-2)}{\Delta_{(1)}^2} \log\left(\frac{1}{2.4\delta}\right) \enspace.
\]
However, this contradicts \cref{lemma:LB2}, applied to $\mathcal{A}'$ with $\delta' = 2\delta$. Thus, we have 
\[
\mathbb{P}_{\nu,\mathcal{A}}\left(T \geqslant \frac{2(n-2)}{\Delta_{(1)}^2} \log\left(\frac{1}{4.8\delta}\right)\right) \geqslant \delta \enspace.
\]
\end{proof}

\begin{proof}[Proof of \cref{lemma:LB2}]
Let $\mathcal{A}$ be any $\delta$-PAC algorithm for the clustering problem, and consider the matrix $M$ that parametrizes the Gaussian environment $\nu$. Recall that $g \in \{0, 1\}^n$ denotes the vector of labels encoding the true partition of the rows of $M$. Without loss of generality, we assume that $g(1) = 0$ and $g(2) = 1$. This is justified if we assume that the algorithm knows one item from each group via an oracle.

For the Gaussian environment $\nu$, let $i, j \in [n] \times [d]$. The observations follow a Gaussian distribution:
\[
\nu_{i,j} = \begin{cases} 
    \mathcal{N}(\mu^a_j, 1), & \text{if } g(i) = 0, \\
    \mathcal{N}(\mu^b_j, 1), & \text{if } g(i) = 1 \enspace.
\end{cases}
\]
We aim to show that with a budget smaller than $\frac{cn}{\Delta_{(1)}^2} \log(1/\delta)$, a $\delta$-PAC algorithm cannot distinguish the environment $\nu$ from another environment where one item from $\nu$ has been switched to the other group. We construct now this alternative environment. 

For any $k \in \{3, \dots, n\}$, define $g^k$ as the vector of labels obtained from $g$ by flipping the label of row $k$, and let $\nu^k$ denote the corresponding Gaussian environment. The lower bound follows from the information-theoretic cost of distinguishing $\nu$ from any $\nu^k$.

To handle multiple environments, let $\mathbb{P}_{g^k}$ (resp. $\mathbb{P}_g$) denote the probability distribution induced by the interaction between algorithm $\mathcal{A}$ and environment $\nu^k$ (resp. $\nu$).

For any $k \in \{3, \dots, n\}$, note that environments $\nu$ and $\nu^k$ differ only on row $k$. By decomposing the KL divergence and using the Gaussian KL formula, we have:
\begin{align} \label{eq:LB2a}
    \KL(\mathbb{P}_g, \mathbb{P}_{g^k}) &= \sum_{j=1}^d \mathbb{E}_g[T_{k,j}] \frac{\Delta_j^2}{2} 
    \leqslant \sum_{j=1}^d \mathbb{E}_g[T_{k,j}] \frac{\Delta_{(1)}^2}{2} \enspace,
\end{align}
where $|\Delta_{(1)}| = \max_{j \in [d]} |\mu^a_j - \mu^b_j|$, and $T_{k,j}$ denotes the number of samples taken from row $k$ and column $j$.

Since $\mathcal{A}$ is $\delta$-PAC for the clustering task, we have:
\begin{align*}
    \mathbb{P}_g(\hat{g} \neq g) &\leqslant \delta, & \mathbb{P}_{g^k}(\hat{g} \neq g^k) &\leqslant \delta \enspace.
\end{align*}
Now, if $\delta \in (0, 1/2)$, by the monotonicity of the binary KL divergence $\kl$, and using the data-processing inequality, we obtain:
\begin{align} \label{eq:LB2b}
    \kl(\delta, 1-\delta) &\leqslant \kl\big(\mathbb{P}_g(\hat{g} = g^k), \mathbb{P}_{g^k}(\hat{g} = g^k)\big) \leqslant \KL(\mathbb{P}_g, \mathbb{P}_{g^k}) \enspace.
\end{align}

Combining \cref{eq:LB2a} and \cref{eq:LB2b}, and summing over $k \in \{3, \dots, n\}$, we get:
\begin{align} \label{eq:LB2c}
    (n-2) \kl(\delta, 1-\delta) &\leqslant \sum_{k=3}^n \sum_{j=1}^d \mathbb{E}_g[T_{k,j}] \frac{\Delta_{(1)}^2}{2} 
    \leqslant \mathbb{E}_g[T] \frac{\Delta_{(1)}^2}{2} \enspace.
\end{align}
Finally, \cref{lemma:LB2} follows by combining \cref{eq:LB2c} with the inequality $\kl(\delta, 1-\delta) \geqslant \log\left(\frac{1}{2.4\delta}\right)$.
\end{proof}

\section{Technical Results}
\begin{lemma}[Chernoff-Bound for Binomial random variables]\label{lem:cernov}
	For $i=1,\dots, n$, consider $X_1,X_2,\dots,X_n\sim^{\mathrm{i.i.d.}}\mathrm{Bern}(p)$ with $p\in(0,1)$, denote $\mu\coloneqq np$ and consider $\kappa>0$. We have
	\begin{align*}
		\mathbb{P}\left(\sum_{i=1}^nX_i\geq (1+\kappa)\mu\right)\leq \frac{e^{\kappa \mu}}{(1+\kappa)^{(1+\kappa)\mu}}.
	\end{align*}
	If $\kappa \in (0,1)$, we also have 
	\begin{align*}
		\mathbb{P}\left(\sum_{i=1}^nX_i\leq (1-\kappa)\mu\right)\leq \exp\left(-\frac{\kappa^2\mu }{2}\right).
	\end{align*}
\end{lemma}

\end{document}